\documentclass{article} % For LaTeX2e
\usepackage{iclr2027_conference,times}

\usepackage{amsmath,amsfonts,bm}

\def\eqref#1{equation~\ref{#1}}
\def\1{\bm{1}}

\def\vb{{\bm{b}}}

\def\vm{{\bm{m}}}

\def\vu{{\bm{u}}}
\def\vv{{\bm{v}}}
\def\vw{{\bm{w}}}
\def\vx{{\bm{x}}}
\def\vy{{\bm{y}}}
\def\vz{{\bm{z}}}

\def\mA{{\bm{A}}}
\def\mB{{\bm{B}}}
\def\mC{{\bm{C}}}
\def\mD{{\bm{D}}}
\def\mE{{\bm{E}}}

\def\mG{{\bm{G}}}
\def\mH{{\bm{H}}}
\def\mI{{\bm{I}}}

\def\mO{{\bm{O}}}
\def\mP{{\bm{P}}}
\def\mQ{{\bm{Q}}}

\def\mS{{\bm{S}}}

\def\mU{{\bm{U}}}
\def\mV{{\bm{V}}}
\def\mW{{\bm{W}}}
\def\mX{{\bm{X}}}

\def\mZ{{\bm{Z}}}

\DeclareMathAlphabet{\mathsfit}{\encodingdefault}{\sfdefault}{m}{sl}
\SetMathAlphabet{\mathsfit}{bold}{\encodingdefault}{\sfdefault}{bx}{n}

\usepackage{hyperref}
\usepackage{url}

\usepackage{amsmath}
\usepackage{amssymb}
\usepackage{mathtools}
\usepackage{algorithm}
\usepackage{algpseudocodex} % for more advanced typesetting
\usepackage[capitalize,noabbrev]{cleveref}
\usepackage{soul}
\usepackage{xspace}
\newcommand{\method}{LSP\xspace}
\newcommand{\methodlong}{Learnable Subspace Projections\xspace}
\newcommand{\myparagraph}[1]{\textbf{#1}}
\usepackage{float}
\usepackage{booktabs}
\usepackage{multirow}
\usepackage{bbm}

\usepackage{pifont}

\usepackage{amsthm}
\newtheorem{proposition}{Proposition}
\usepackage{enumitem}
\usepackage{caption}
\usepackage{wrapfig}

\title{Learning Functional Subspaces for\\Neural Network Compression}

\author{
Massimo Bini$^{1,2,3}$ \,\,Anders Christensen$^4$ \,\,Stephan Alaniz$^5$ \,\,Judah Goldfeder$^6$\\
\,\,\textbf{Ole Winther$^{7,8}$ \,\,Yann LeCun$^9$ \,\,Ravid Shwartz-Ziv$^9$ \,\,Zeynep Akata$^{1,2,3}$}
\And\vspace{-0.5cm}\\
$^1$Helmholtz Munich, $^2$Technical University of Munich, $^3$MCML, $^4$Orbital Industries\\
$^5$LTCI, Télécom Paris, Institut Polytechnique de Paris, $^6$Columbia University, \\
$^7$University of Copenhagen, $^8$Technical University of Denmark,
$^9$New York University
}

\iclrfinalcopy % Uncomment for camera-ready version, but NOT for submission.
\begin{document}

\maketitle
\begin{abstract}
Modern transformers pair impressive capabilities with substantial memory and compute demands.
Low-rank weight factorization reduces both while keeping the matrices dense, and thus efficient on standard hardware.
Existing methods, however, choose the subspace to remove from each weight matrix with \emph{local} closed-form criteria: activation energy, layer-wise reconstruction error, or a quadratic approximation of the loss.
These criteria ignore how errors propagate through the network, so at high compression the errors compound with depth and performance collapses.
We introduce \emph{Learnable Subspace Projections} (LSP), which instead learns the subspaces to discard end-to-end.
Each linear layer, or \emph{tied group} of layers that read the same activations, is assigned an orthogonal projector.
All projectors are optimized jointly against a \emph{global} objective---the KL divergence to the dense model's output distribution or the model's original training loss---while the pretrained weights remain frozen.
Projectors are initialized from a whitened SVD truncation, and ranks are allocated by the output KL each projector induces per parameter saved.
After training, the projectors merge into standard low-rank factors, with each tied group sharing one factor.
In attention, this also lets the model cache one narrow latent in place of full keys and values.
Across LLMs (OPT-125M/1.3B, Qwen3-4B, Llama-2-7B) and ViT-B/16, LSP outperforms baselines, and its advantage widens as compression increases.
At $-70\%$ compression, LSP brings Llama-2-7B to $10.9$ WikiText-2 perplexity and $42.2\%$ mean zero-shot accuracy, versus $13.3$ and $36.0\%$ for the strongest baseline.
The factorized model decodes up to $1.6\times$ faster than the dense model at small batch sizes, and aching the shared latent shrinks the combined memory of weights and KV cache by $13.5\times$ at a 128k-token context, versus at most $6.5\times$ for untied baseline factorizations.
\end{abstract}

\section{Introduction}

Transformers \citep{vaswani2017attention} have become the dominant architecture in deep learning, owing to their flexibility, scalability, and strong performance across tasks and modalities.
Their uniform design enables rapid development and stable scaling; however, it also leaves trained models highly redundant. 
Large models exhibit low intrinsic dimensionality \citep{aghajanyan-etal-2021-intrinsic,kuzborskij2025lowrankbiasweightdecay}, with their capabilities concentrated in a low-dimensional subspace of their parameters.
Compression exploits this redundancy to reduce memory and compute while preserving the model's behavior.
\smallskip\\
Low-rank compression is a natural way to realize these savings, replacing each weight matrix with the product of two smaller factors.
Yet this redundancy is not directly visible in the weights' spectra.
The singular values of pretrained Transformer weights decay slowly, so naive truncated SVD severely degrades performance \citep{hsu2022language}.
Practical low-rank compression instead builds on an empirical observation: the \emph{activations} flowing between layers lie in low-dimensional subspaces \citep{lowrankfeats,yuan2024asvdactivationawaresingularvalue,garg2024revealing,skean2025layer}.
Existing methods read the useful subspace off activation statistics, either from the energy of the layer input, as in ASVD \citep{yuan2024asvdactivationawaresingularvalue}, SliceGPT \citep{ashkboos2024slicegpt}, and MoDeGPT \citep{lin2025modegptmodulardecompositionlarge}, or from the reconstruction error of the layer output, as in SVD-LLM \citep{wang2025svdllm} and Swift-SVD \citep{qi2026swiftsvd}.
Either criterion captures the layer, not the network: a low-energy direction is discarded whether or not the network's output depends on it (\Cref{fig:teaser}), and a layer-wise optimum says nothing about how its residual error propagates downstream.
Loss-aware methods bring the network in through a local surrogate---Fisher-weighted reconstruction in FW-SVD \citep{hsu2022language} or a quadratic curvature model in LLM-Surgeon \citep{van2023llm}---but such surrogates hold only for small perturbations of the dense weights, whereas aggressive compression moves far from them.
Dobi-SVD \citep{qinsi2025dobisvd} optimizes the loss directly, but learns only the per-matrix truncation rank.
The loss thus shapes local approximations, sets ranks, or repairs the weights after truncation; in none of these methods does the network loss itself choose which directions are removed.
What compression has to identify, then, is the network's \emph{functional subspace}: the directions whose retention preserves its behavior on the data of interest.
Our starting observation is that choosing a subspace amounts to choosing an orthogonal projector: applying a rank-$r$ projector to a weight matrix leaves it with rank at most $r$, so it factors into two thin matrices.
The search for the functional subspace can thus be cast as a search over projectors.
This raises our central question: \emph{can we obtain better compressed models by learning which subspace to remove, optimizing the projectors end-to-end against the frozen network's output?}
\begin{figure*}[t]
\centering
\vspace{-0.2cm}
\includegraphics[width=\linewidth]{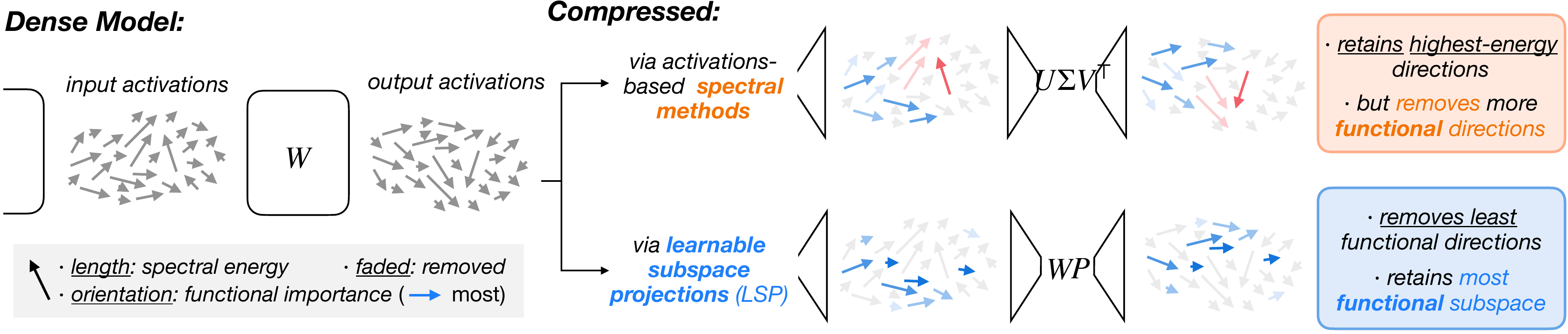}
\vspace{-0.3cm}
\caption{
\textbf{Spectral energy is not function.}
Each arrow is a direction in the input activation space of a linear layer $\mW$: its length is the spectral energy it carries, its orientation how much the network output depends on it (horizontal right: most), and faded arrows are removed.
Activation-based low-rank compression (top) keeps the highest-energy directions, even when the output barely depends on them (red), and can discard functional ones (blue).
\method (bottom) learns the projector $\mP$ against the network output, removing the directions the output is least sensitive to, whatever their energy.
}
\vspace{-0.4cm}
\label{fig:teaser}
\end{figure*}
\smallskip\\
We propose \methodlong (\method), which does exactly this: every projector is learned against a global objective, by default the KL divergence to the dense model's output distribution, or alternatively the model's original training objective, a variant we denote \method$^T$.
Our contributions are:\\
(1) \textbf{Learned subspace projections.} Layers that read the same activations (Q/K/V, gate/up) form a \emph{tied group} and share one projector; we optimize all projectors jointly across the network with the pretrained weights frozen. An activation-space parameterization with fused QR orthonormalization keeps this tractable at billion-parameter scale.
\\%
(2) \textbf{Whitened initialization and measured-KL allocation.} Each projector is initialized from a whitened truncation, recast as an orthogonal projector and extended to tied groups.
To decide how many directions each layer or tied group removes, we apply each projector in isolation, measure the KL divergence it induces on the model output, and allocate the budget by KL cost per saved parameter.
The model at initialization, which we call NoLSP, serves as a control that isolates the effect of learning.
\\%
(3) \textbf{Efficient low-rank inference.} After training, the projectors merge into plain low-rank factors, so the deployed model contains no \method-specific operations.
At the same compression ratio it requires the same FLOPs per token as other factorizations, yet the input factor shared by each tied group speeds up small-batch decoding and, in a latent-cache deployment, stores one narrow latent per group where untied factorizations store two (one each for keys and values), shrinking the KV cache.
\\%
Because \method applies to any linear layer, the same recipe transfers across Transformer models without architecture-specific surgery.
\method outperforms pruning and SVD baselines by a margin that grows with the compression ratio, and on ViT it degrades least among the compared methods when calibration and evaluation distributions differ.
\vspace{-0.1cm}

\section{\methodlong (\method)}

\method compresses a pretrained network into its functional subspace: each targeted linear layer, or tied group of layers that read the same activations, is composed with a learned orthogonal projector, and all projectors are optimized end-to-end against the model output with the pretrained weights frozen.
\Cref{fig:pipeline} summarizes the pipeline, \Cref{sec:lsp} gives the formulation, \Cref{subsec:alloc_init} the whitened initialization and measured-KL allocation, and \Cref{subsec:training_loop} the training procedure and the merge into low-rank factors.

\subsection{Learning Subspace Projections}
\label{sec:lsp}

\myparagraph{Projector parameterization.}
Consider a linear layer $\vy = \mW\vx + \vb$ with $\mW \in \mathbb{R}^{d_\text{out} \times d_\text{in}}$; the bias is left unchanged throughout.
To remove $k$ input directions, we learn $\mU \in \mathbb{R}^{d_\text{in} \times k}$ with orthonormal columns and apply the orthogonal projector $\mP = \mI - \mU\mU^\top$ to the weight:\vspace{-0.07cm}
\begin{equation}
    \mW \mP \;=\; \mW - (\mW\mU)\mU^\top,
    \label{eq:lsp_proj}\vspace{-0.07cm}
\end{equation}
which, as a result, has rank at most $d_\text{in} - k$.
Orthonormality holds by construction: we optimize an unconstrained $\mV \in \mathbb{R}^{d_\text{in} \times k}$ and set $\mU = \operatorname{qf}(\mV)$, the orthogonal factor of its thin QR decomposition, on every forward pass.
Since $\mP$ is invariant under $\mU \mapsto \mU\mO$ for any orthogonal $\mO$, the effective variable is the removed subspace $\operatorname{span}(\mU)$, which equals $\operatorname{span}(\mV)$ whenever $\mV$ has full column rank.
The output-side form is $\mP\mW$, with $\mU \in \mathbb{R}^{d_\text{out} \times k}$ and rank at most $d_\text{out} - k$.

\begin{figure*}[t]
\centering
\vspace{-0.2cm}
\includegraphics[width=\linewidth]{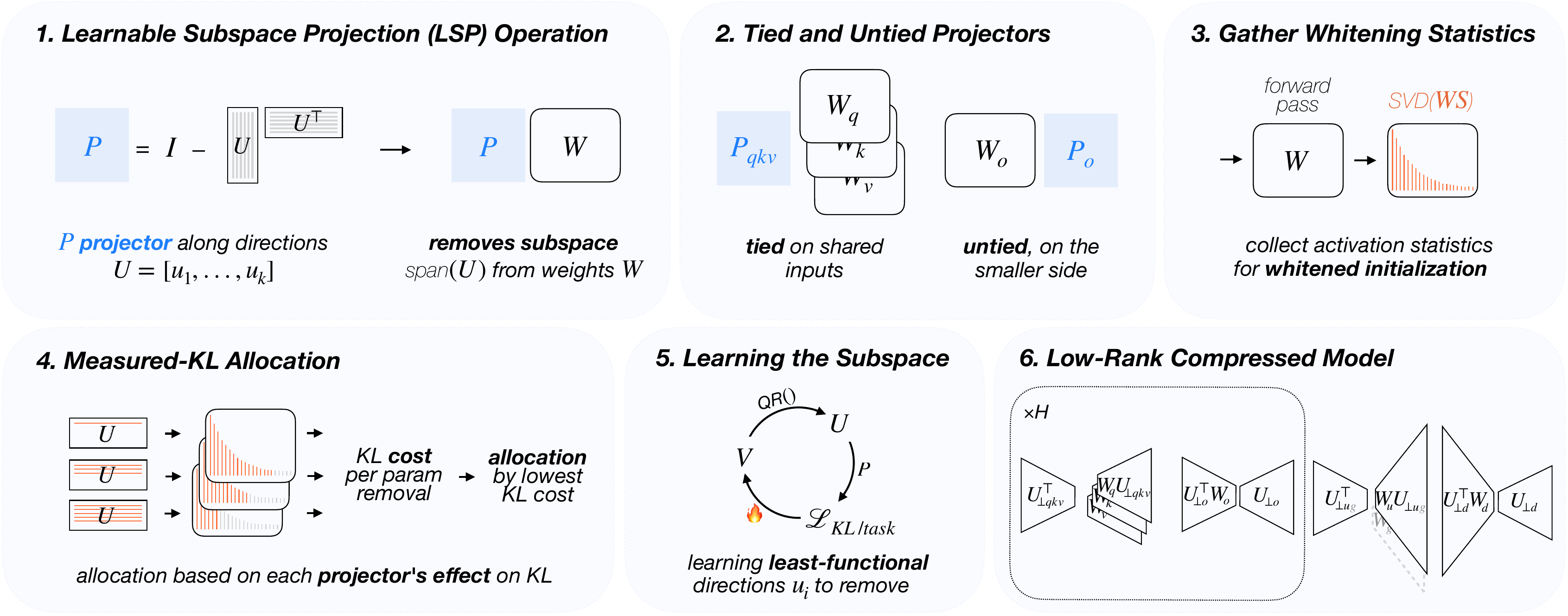}
\vspace{-0.5cm}
\caption{\textbf{Overview of \method.}
(1)~An orthogonal projector $\mP = \mI - \mU\mU^\top$ removes the learned subspace $\operatorname{span}(\mU)$ from the input or output of $\mW$.
(2)~Layers that read the same activations (Q/K/V, gate/up) share a tied projector; the remaining (untied) layers are projected on their smaller side.
\textbf{(3)} Calibration activations give the whitened SVD of $\mW\!\mS$, whose trailing directions initialize $\mU$.
\textbf{(4)} The KL divergence that each candidate truncation induces on the model output, per saved parameter, sets how many directions each unit removes.
\textbf{(5)} The unconstrained $\mV$ is trained, with $\mW$ frozen, on the KL to the dense model or on the task loss; $\mU \!=\!\operatorname{qf}(\mV)$ is its orthonormal QR factor, so the removed directions are those the output depends on least.
(6)~The trained projectors decompose into low-rank factors $\mU_\perp\mU_\perp^\top$, and are merged with the frozen weights.}
\vspace{-0.25cm}
\label{fig:pipeline}
\end{figure*}

\myparagraph{Tied groups and projection side.}
Q/K/V read the same normalized activation, as do gate/up; each such \emph{tied group} shares a single \emph{tied projector} on this common input.
Every other layer is projected on its smaller side, where, for a full-rank weight, each removed direction lowers the rank by one and $\mU$ is smallest.
Sharing lets the group retain a higher rank at the same parameter budget and, tying K/V on the input side, cache a single latent instead of full keys and values (\Cref{subsec:training_loop}).

\myparagraph{Practical design choices.} Three choices make training scalable and stable; details are in \Cref{app:design_choices}.
\emph{(i) Activation-space form.} We compute $\mW(\vx - \mU(\mU^\top \vx))$ without forming $\mW\mP$, which would materialize a dense $d_\text{out} \times d_\text{in}$ matrix and its gradient per layer, and a fused QR backward stores $\mathcal{O}(d_\text{in}k + k^2)$ per layer, versus $\mathcal{O}(d_\text{in}k^2)$ for an explicit Gram--Schmidt loop.
\emph{(ii) Warm-up and direction dropout.} Training uses $\mP_{\alpha,\vm} = \mI - \alpha\,\mU\operatorname{diag}(\vm)\mU^\top$: $\alpha$ ramps from $0$ to $1$ over the first epoch to avoid abrupt output changes, and a random mask $\vm \in \{0,1\}^k$ temporarily keeps some removed directions, acting as dropout.
\emph{(iii) Orthogonality penalty.} A penalty $\mathcal{L}_\text{ort}$ on correlations between the columns of $\mV$ (\Cref{eq:orth_reg}) stabilizes QR and its gradient.

\subsection{Whitened Initialization and Measured-KL Allocation}
\label{subsec:alloc_init}
Before training, each compressed unit---a tied group or an individual layer---needs a subspace to start from and a number of directions to remove.
The subspace comes from an ordered whitened basis, so that removing $k$ directions amounts to dropping its trailing $k$ vectors; the number comes from the KL divergence each candidate truncation induces on the model output, allocated under the global parameter budget.
The basis is thus spectral, the allocation loss-aware.

\myparagraph{Whitened initialization.}
We initialize from a whitened truncation, which scores directions by their effect on the layer output, as introduced by \citet{wang2025svdllm}, and extend it in two ways: from a per-matrix truncation to a tied group, and from an operator that is orthogonal only on the output side to a projector on either side.
Dropping the layer index, let $\mX \in \mathbb{R}^{n \times d_\text{in}}$ hold $n$ calibration inputs, with positive-definite Gram matrix $\mG = \tfrac{1}{n}\mX^\top\mX = \mS\mS^\top$ ($\mS$ its Cholesky factor), so that a candidate weight $\mZ$ has reconstruction error $\mathcal{E}(\mZ) = \tfrac{1}{n}\|(\mW - \mZ)\mX^\top\|_F^2 = \|(\mW - \mZ)\mS\|_F^2$.
The whitened truncation minimizes it at rank $r$ by keeping the leading $r$ singular directions of $\mW\mS = \mU^{w} \boldsymbol{\Sigma}^{w} (\mV^{w})^\top$ and discarding the trailing $k = d - r$, $d$ being the dimension of the projected side; subscripts $r$ and $>\!r$ denote the corresponding blocks of columns:\vspace{-0.05cm}
\begin{equation}\vspace{-0.05cm}
    \widehat{\mW} \;=\; \mU^w_r(\mU^w_r)^\top\mW \;=\; \mW\,\mS\mV^w_r(\mV^w_r)^\top\mS^{-1},
    \qquad
    \mathcal{E}(\widehat{\mW}) = \textstyle\sum_{i > r} (\sigma^{w}_i)^2 .
    \label{eq:whitening}
\end{equation}
The two forms are the same matrix.
The first is an orthogonal projector applied on the output side, and we use it as is.
The second, applied on the input side, keeps $\mathcal{K} = \operatorname{span}(\mS\mV^w_r)$ and zeroes $\operatorname{span}(\mS\mV^w_{>r})$, two subspaces that are not orthogonal unless $\mG \propto \mI$, so it is oblique; there we replace it by the orthogonal projector onto $\mathcal{K}$, which removes the orthogonal complement $\operatorname{span}(\mS^{-\top}\mV^w_{>r})$.

\begin{proposition}[Orthogonal recast]
\label{prop:whitened}
Let $\mC_r = \mS\mV^w_r$, $\mC_{>r} = \mS\mV^w_{>r}$,\vspace{-0.05cm} and $\mP_\text{init} = \mC_r\mC_r^+$ be the orthogonal projector onto $\mathcal{K}$.
\emph{(i)} On the output side, $(\mI - \mU^w_{>r}(\mU^w_{>r})^\top)\mW = \widehat{\mW}$ attains the optimum $\sum_{i>r}(\sigma^w_i)^2$.
\emph{(ii)} On the input side, $\mathcal{E}(\mW\mP_\text{init}) = \sum_{i>r}(\sigma^w_i)^2 + \|\boldsymbol{\Sigma}^w_r\,\mC_r^+\mC_{>r}\|_F^2$, and if $\sigma^w_r > 0$ the excess term vanishes if and only if $\mC_r^\top\mC_{>r} = 0$.
\emph{(iii)} For a group tied on its input, the summed error under one shared kept subspace equals that of the row-stacked whitened weight $\bar\mW\mS$, whose leading right singular subspace minimizes it, and (ii) applies to $\bar\mW$; a group tied on its output is symmetric, with the column-stacked weight, its leading left singular subspace, and (i).
\end{proposition}\vspace{-0.25cm}
\smallskip
The proof is in \Cref{app:whitened_excess}.
The initialized layer thus agrees with the whitened truncation on the kept subspace $\mathcal{K}$, while on directions $\mC_{>r}$, which the truncation zeroes, it retains their least-squares component along $\mC_r$; the excess is that component weighted by the kept singular values, and it vanishes when the whitened truncation is itself orthogonal, as for an isotropic input Gram matrix.

\myparagraph{Measured-KL rank allocation.}
Units differ in how much the model output depends on them, so we measure each unit's sensitivity directly and remove parameters where the measured cost per saved parameter is lowest; sensitive units keep more rank or stay dense.
For a unit $u$, removing the trailing $k$ directions of its initialized basis, with the rest of the network dense, costs\vspace{-0.05cm}
\begin{equation}\vspace{-0.05cm}
    \Delta_u(k) \;=\; \mathbb{E}_{\vx\sim\mathcal{D}_\text{alloc}}\,
    D_\mathrm{KL}\big(p_{\theta_0}(\cdot\mid\vx)\,\big\|\,p_{\theta_0\setminus(u,k)}(\cdot\mid\vx)\big),
    \label{eq:kl_cost}
\end{equation}
where $\mathcal{D}_\text{alloc}$ is the subset of the calibration set $\mathcal{D}_\text{cal}$ used for allocation and $\theta_0\!\setminus\!(u,k)$ denotes the dense model with only this truncation applied. 
Measuring $\Delta_u$ needs no gradients and no new SVDs: the dense model's outputs are cached once, and every candidate reuses the unit's ordered basis.
We measure it on a grid of removal fractions and interpolate the running maximum of each curve, so that the resulting estimate $\tilde\Delta_u$ is nondecreasing in $k$.
Starting from the dense model, the allocator repeatedly takes the move $k\! \to\! k'$ with the smallest marginal cost \vspace{-0.05cm}$\big(\tilde\Delta_u(k') - \tilde\Delta_u(k)\big)/\big(s_u(k') - s_u(k)\big)$, where $s_u(k)$ is the number of parameters the unit saves after factorization, counting shared factors once, until the target savings are reached (\Cref{app:alloc_details}).
Unlike curvature-based estimates \citep{van2023llm}, which hold only near the dense weights, this measures finite truncations through the full network, with joint training handling the combined effect of independently measured units.

\subsection{Training Procedure}
\label{subsec:training_loop}\vspace{-0.1cm}
We replace each targeted linear layer with its \method counterpart and initialize it as in \Cref{subsec:alloc_init}: we construct the whitened bases, allocate ranks by measured KL, and set $\{\mV_\ell\}$ accordingly.
We then optimize all $\{\mV_\ell\}$ jointly at fixed ranks, with every pretrained parameter frozen, against\vspace{-0.05cm}
\begin{equation}\vspace{-0.05cm}
    \mathcal{L}_\text{total} = \mathcal{L}_\text{obj}\big(f\big(\cdot\,; \{\mP^{(\ell)}_{\alpha,\vm}\}\big)\big) + \lambda_\text{ort}\,\mathcal{L}_\text{ort}.
\end{equation}
By default, $\mathcal{L}_\text{obj}$ is the output-distillation loss \citep{hinton2015distilling}%\vspace{-0.05cm}
\begin{equation}%\vspace{-0.05cm}
    \mathcal{L}_\text{KL} = \mathbb{E}_{\vx \sim \mathcal{D}_\text{cal}}\, \mathrm{KL}\big(p_\text{dense}(\cdot \mid \vx)\,\big\|\,p_\text{compressed}(\cdot \mid \vx)\big),
    \label{eq:kl_loss}
\end{equation}
where $p$ is the distribution over the next token for LLMs, averaged over all positions, and over the last Transformer block's output for the ViT, averaged over tokens.
Alternatively, $\mathcal{L}_\text{obj}$ is the model's original training loss (next-token or classification cross-entropy), a variant we denote \method$^T$.
The teacher is the same network with its projectors disabled ($\alpha = 0$), run without gradient tracking, so no second copy of the weights is stored.
We use a linear-warmup cosine learning-rate schedule with early stopping on a held-out validation split, and keep the best validation epoch (\Cref{app:protocol}).

\myparagraph{Choice of objective.}
Distillation is the default: it targets the dense model's full output distribution rather than a single target, and for classifiers it needs no labels, so any unlabeled pool can serve for calibration (\Cref{subsec:exp_vit}).
The task loss (\method$^T$) instead fits the calibration data directly, so it can specialize the model when the calibration data matches the deployment task; in our experiments, \method$^T$ reaches its best validation score in fewer epochs on average (\Cref{app:cost}).

\myparagraph{Merging.}
After training, each projector is folded into its weight, a step we call \emph{merging}.
For an input-side projector, let the columns of $\mU_\perp \in \mathbb{R}^{d_\text{in} \times r}$, with $r = d_\text{in} - k$, form an orthonormal basis of $\operatorname{span}(\mU)^\perp$.
Then $\mP = \mU_\perp\mU_\perp^\top$, and the merged weight factorizes \emph{exactly}:
\begin{equation}
    \mW \mP \;=\; (\mW\mU_\perp)\,\mU_\perp^\top \;=\; \mB\mA,
    \qquad
    \mA \;=\; \mU_\perp^\top \in \mathbb{R}^{r \times d_\text{in}},
    \quad
    \mB \;=\; \mW\mU_\perp \in \mathbb{R}^{d_\text{out} \times r},
    \label{eq:merge}
\end{equation}
a rank-$r$ factorization with $(d_\text{in}\!+\!d_\text{out})\,r$ parameters; on the output side, $\mP\mW \!=\! \mU_\perp(\mU_\perp^\top\mW)$.
Every member of an input-tied group shares $\mA$, which is stored and applied once; in attention, the common latent $\vz = \mA\vx$ suffices to reconstruct both keys and values, so the KV cache can store $\vz$ alone.

\section{Experiments}
\label{sec:experiments}

%%%%%%%%%
\begin{wraptable}{r}{0.5\linewidth}
\centering
\vspace{-0.42cm}
\caption{Wikitext-2 perplexity ($\downarrow$) of LLMs at $-30/50/70\%$.
One of the two \method variants has the lowest perplexity in all 12 settings, by a margin that grows with the ratio; at $-70\%$ both variants beat every baseline on all four models.
}
\vspace{-0.3cm}
\label{tab:results}
\scriptsize
\setlength{\tabcolsep}{0.5pt}
\renewcommand{\arraystretch}{0.92}
\resizebox{\linewidth}{!}{
\begin{tabular}{@{}c@{\hspace{1pt}}c@{\hspace{4pt}}lcccc@{}}
\toprule
\textbf{Comp.} & & \textbf{Method} & \textbf{OPT-125M} & \textbf{OPT-1.3B} & \textbf{Qwen3-4B} & \textbf{Llama2-7B} \\
\midrule
\multicolumn{3}{@{}l}{Dense (Base)} & 27.9 & 14.7 & 8.0 & 5.5 \\
\midrule
\multirow{10}{*}{-30\%} & \multirow{5}{*}{\rotatebox[origin=c]{90}{\itshape \shortstack{activations/\\[-2.5pt]reconstruc.}}} & NoLSP & 48.5 & 23.0 & 16.8 & 9.4 \\
 &  & ASVD & 106.8 & 396.1 & 154.5 & 98.2 \\
 &  & SliceGPT & 51.7 & 23.9 & 15.0 & 10.3 \\
 &  & SVD-LLM (W) & 48.7 & 20.2 & 19.1 & 10.4 \\
 &  & Swift-SVD & 47.0 & 19.8 & 15.6 & 10.4 \\
\addlinespace[2pt]
 & \multirow{5}{*}{\rotatebox[origin=c]{90}{\itshape \shortstack{\\loss-aware}}} & LLM-Surgeon & 31.7 & 16.2 & 11.2 & 8.9 \\
 &  & Dobi-SVD & 63.2 & 23.8 & 18.5 & 10.0 \\
 &  & SVD-LLM & 31.1 & \underline{13.9} & 10.3 & \underline{6.5} \\
 &  & \textbf{LSP$^T$ (Ours)} & \textbf{24.5} & \textbf{12.9} & \underline{9.2} & \textbf{6.4} \\
 &  & \textbf{LSP (Ours)} & \underline{31.0} & 15.7 & \textbf{9.1} & 6.6 \\
\midrule
\multirow{10}{*}{-50\%} & \multirow{5}{*}{\rotatebox[origin=c]{90}{\itshape \shortstack{activations/\\[-2.5pt]reconstruc.}}} & NoLSP & 206.9 & 81.2 & 43.5 & 29.8 \\
 &  & ASVD & 1463.2 & 5213.1 & 3586.5 & ${>}10^{4}$ \\
 &  & SliceGPT & 97.2 & 57.8 & 32.0 & 25.2 \\
 &  & SVD-LLM (W) & 126.5 & 44.2 & 82.1 & 30.2 \\
 &  & Swift-SVD & 114.9 & 39.5 & 41.4 & 29.4 \\
\addlinespace[2pt]
 & \multirow{5}{*}{\rotatebox[origin=c]{90}{\itshape \shortstack{\\loss-aware}}} & LLM-Surgeon & 47.3 & 22.5 & 25.8 & 17.8 \\
 &  & Dobi-SVD & 183.1 & 54.0 & 39.3 & 28.1 \\
 &  & SVD-LLM & 54.5 & 18.7 & 14.1 & \underline{8.4} \\
 &  & \textbf{LSP$^T$ (Ours)} & \textbf{30.6} & \textbf{16.4} & \underline{12.5} & \underline{8.4} \\
 &  & \textbf{LSP (Ours)} & \underline{35.4} & \underline{18.5} & \textbf{11.4} & \textbf{8.0} \\
\midrule
\multirow{10}{*}{-70\%} & \multirow{5}{*}{\rotatebox[origin=c]{90}{\itshape \shortstack{activations/\\[-2.5pt]reconstruc.}}} & NoLSP & 1899 & 922.8 & 1274 & 222.8 \\
 &  & ASVD & 3849.0 & 9080.1 & 8991.4 & ${>}10^{5}$ \\
 &  & SliceGPT & 290.1 & 262.7 & 89.0 & 86.3 \\
 &  & SVD-LLM (W) & 633.4 & 590.0 & 2246.7 & 214.9 \\
 &  & Swift-SVD & 561.0 & 412.8 & 1458.5 & 203.6 \\
\addlinespace[2pt]
 & \multirow{5}{*}{\rotatebox[origin=c]{90}{\itshape \shortstack{\\loss-aware}}} & LLM-Surgeon & 113.7 & 53.9 & 371.0 & 105.0 \\
 &  & Dobi-SVD & 759.6 & 338.3 & 92.5 & 61.2 \\
 &  & SVD-LLM & 187.6 & 36.0 & 22.5 & 13.3 \\
 &  & \textbf{LSP$^T$ (Ours)} & \textbf{42.4} & \underline{23.3} & \underline{19.4} & \underline{13.0} \\
 &  & \textbf{LSP (Ours)} & \underline{48.9} & \textbf{22.4} & \textbf{16.2} & \textbf{10.9} \\
\bottomrule
\vspace{-0.7cm}
\end{tabular}}

\end{wraptable}
%%%%%%%%%
We evaluate \method up to $-70\%$ compression on decoder language models (\Cref{subsec:exp_llm}) and on vision transformers (\Cref{subsec:exp_vit}), then measure inference efficiency (\Cref{subsec:compute}) and analyze the subspaces \method learns (\Cref{subsec:analysis}).

\textbf{Setup.} For language models we compress OPT-125M/1.3B \citep{zhang2022opt}, Qwen3-4B \citep{yang2025qwen3technicalreport} and Llama-2-7B \citep{touvron2023llama2openfoundation}. We report WikiText-2 \citep{merity2016pointer} test perplexity, calibrating as SliceGPT does on 1024 sequences of 2048 tokens from the training split, and zero-shot accuracy on six benchmarks with the calibration data given in \Cref{subsec:exp_llm}.
Early stopping and every tuned setting are selected on held-out validation splits.
On Qwen3-4B, whose grouped-query K and V are a quarter of Q's width, K and V share an output-side projector and Q is projected alone; every other model ties Q/K/V on the input side (\Cref{app:protocol}).
For vision we compress a ViT-B/16 \citep{dosovitskiy2020image} fine-tuned on CIFAR-100 \citep{krizhevsky2009learning} and report source accuracy and linear-probe transfer.
A compression ratio is the fraction of the parameters of the linear layers (no embeddings, no head) that is removed.
All baselines run from their official code with the same calibration data and precision as \method; SVD baselines are evaluated in factorized form, with tunable settings swept around the official optima (\Cref{app:protocol}).

\textbf{Baselines.} \emph{Activation- and reconstruction-based baselines}, constructed from activation statistics or reconstruction criteria: SliceGPT \citep{ashkboos2024slicegpt}, ASVD \citep{yuan2024asvdactivationawaresingularvalue}, SVD-LLM~(W) \citep{wang2025svdllm}, whose whitening minimizes the layer-wise reconstruction error, and Swift-SVD \citep{qi2026swiftsvd}, which reaches the same per-layer optimum from the output covariance and adds a per-matrix rank allocation.
\emph{Loss-aware baselines} use gradients of the loss: LLM-Surgeon \citep{van2023llm}, through a Kronecker-factored curvature model, and Dobi-SVD \citep{qinsi2025dobisvd}, which learns only the per-matrix truncation rank within a closed-form basis, run without its quantization step so that parameter counts match.
SVD-LLM is reported also with its LoRA recovery fine-tuning.
On vision we compare with SliceGPT, SVD-LLM~(W), FLAR-SVD \citep{Thoma_2025_CVPR}, and PELA \citep{guo2024pela}. The last two are proposed on vision models, with PELA retraining all parameters by distilling every block's output features from the dense model.
NoLSP is the non-trained version of LSP, controlling for the learning stage at fixed initialization, ranks and ties.

\vspace{-0.05cm}
\subsection{Large Language Models}
\label{subsec:exp_llm}

\myparagraph{Perplexity.}
\Cref{tab:results} reports WikiText-2 test perplexity at $-30/50/70\%$ compression.
One of the two \method variants has the lowest perplexity in all 12 model--ratio settings, and the margin grows with the ratio.
At $-70\%$ the training-free SVD methods deteriorate sharply, and so does NoLSP, the untrained initialization at the same ranks: learning the directions is what keeps the compressed models usable.
No loss-aware baseline closes the gap at $-70\%$, including LoRA-recovered SVD-LLM, the strongest baseline on the three larger models.
The task loss (\method$^T$) leads mostly on the smaller models and at lower ratios, where it can fall below dense test perplexity (both OPT models at $-30\%$).
Since the calibration text is the WikiText-2 training split, this reflects specialization to the evaluation domain (\Cref{subsec:training_loop}); LoRA-recovered SVD-LLM shows the same effect on OPT-1.3B. Distillation (\method) leads instead on the larger models and at higher ratios.
\smallskip\\
\myparagraph{Zero-shot accuracy.}
Optimizing compression on a target objective raises a natural concern: does the compressed model retain \emph{general} capabilities, or is the retained capacity overfit to the calibration objective?
\Cref{tab:zeroshot_6bench} follows two protocols from the literature on six commonsense, science and math benchmarks, OpenBookQA \citep{OpenBookQA2018}, ARC-easy \citep{clark2018think}, WinoGrande \citep{sakaguchi2019winogrande}, HellaSwag \citep{zellers2019hellaswag}, PIQA \citep{bisk2020piqa} and MathQA \citep{amini2019mathqa}: Llama-2-7B at $-30/50/70\%$ with Alpaca calibration \citep{alpaca}, following \citet{wang2025svdllm}, and Qwen3-4B at $-20/40/60\%$ with C4 calibration \citep{c4_raffel}, following \citet{qi2026swiftsvd}.
We compare with the low-rank baselines, the closest to \method (\Cref{app:protocol}).
Both \method variants lead the strongest training-free baseline in mean accuracy at every ratio, by $6.7$--$10.5$ points on Llama-2-7B and $6.2$--$10.8$ on Qwen3-4B.
The two objectives are within two points of each other, with distillation ahead or tied in every setting.

\begin{table}[t]
\vspace{-0.2cm}
\centering
\caption{Zero-shot accuracy ($\uparrow$, \%). Left: Llama-2-7B at 30/50/70\% compression, every method calibrated on Alpaca. Right: Qwen3-4B at 20/40/60\% compression, C4 calibration. Both \method variants lead every baseline in mean accuracy at every ratio, by a margin that grows with the ratio.
}
\vspace{-0.3cm}
\label{tab:zeroshot_6bench}\label{tab:qwen3_4b}
\scriptsize
\setlength{\tabcolsep}{1pt}
\resizebox{\linewidth}{!}{
\begin{tabular}{@{}lcccccccc@{\hspace{8pt}}cccccccc@{}}
\toprule
 & \multicolumn{8}{c}{\textbf{Llama-2-7B} (Alpaca calibration)} & \multicolumn{8}{c}{\textbf{Qwen3-4B} (C4 calibration)} \\
\cmidrule(lr){2-9}\cmidrule(l){10-17}
\textbf{Method} & \textbf{Comp.} & \textbf{ARC-e} & \textbf{PIQA} & \textbf{Openb.} & \textbf{WinoG.} & \textbf{HellaS.} & \textbf{MathQA} & \textbf{Avg.} & \textbf{Comp.} & \textbf{ARC-e} & \textbf{PIQA} & \textbf{Openb.} & \textbf{WinoG.} & \textbf{HellaS.} & \textbf{MathQA} & \textbf{Avg.} \\
\midrule
Dense (Base) & -- & 75.6 & 77.8 & 32.8 & 69.9 & 57.1 & 28.1 & 56.9 & -- & 79.0 & 77.9 & 32.0 & 70.3 & 54.6 & 53.8 & 61.3 \\
\midrule
SVD-LLM (W) & \multirow{6}{*}{-30\%} & 61.4 & 67.2 & 24.4 & 61.3 & 38.7 & 23.0 & 46.0 & \multirow{6}{*}{-20\%} & 59.0 & 70.5 & 22.4 & 59.4 & 40.1 & 26.5 & 46.3 \\
Swift-SVD   & & 61.6 & 68.1 & 24.0 & 61.5 & 38.9 & 23.0 & 46.2 & & 63.6 & 72.2 & 26.8 & 64.2 & 45.7 & 28.6 & 50.2 \\
Dobi-SVD    & & 53.3 & 64.1 & 20.4 & 57.5 & 35.7 & 22.8 & 42.3 & & 57.5 & 69.0 & 22.6 & 62.9 & 41.5 & 26.7 & 46.7 \\
SVD-LLM     & & 68.5 & 72.7 & 30.2 & 64.7 & 49.0 & 24.3 & \underline{51.6} & & 63.1 & 72.7 & 27.2 & 64.9 & 47.6 & 31.2 & 51.1 \\
\textbf{LSP$^T$ (Ours)} & & 71.3 & 74.6 & 30.2 & 63.3 & 50.7 & 27.3 & \textbf{52.9} & & 77.7 & 77.1 & 29.0 & 65.1 & 51.0 & 38.6 & \underline{56.4} \\
\textbf{LSP (Ours)}     & & 71.2 & 75.4 & 30.0 & 64.0 & 50.4 & 26.4 & \textbf{52.9}    & & 78.3 & 76.7 & 30.8 & 65.6 & 51.8 & 40.0 & \textbf{57.2} \\
\midrule
SVD-LLM (W) & \multirow{6}{*}{-50\%} & 41.0 & 58.4 & 17.4 & 52.2 & 29.6 & 21.9 & 36.8 & \multirow{6}{*}{-40\%} & 36.3 & 59.6 & 13.0 & 50.8 & 29.3 & 22.3 & 35.2 \\
Swift-SVD   & & 41.8 & 58.7 & 19.0 & 52.5 & 29.8 & 22.1 & 37.3 & &50.8 & 66.1 & 18.4 & 58.8 & 34.1 & 23.2 & 41.9 \\
Dobi-SVD    & & 36.4 & 57.7 & 16.2 & 51.7 & 29.3 & 22.0 & 35.6 & & 38.0 & 60.7 & 15.0 & 55.6 & 32.2 & 21.6 & 37.2 \\
SVD-LLM     & & 57.3 & 65.7 & 24.6 & 57.4 & 39.9 & 22.5 & 44.6 & & 58.1 & 67.5 & 21.2 & 58.3 & 39.7 & 23.4 & 44.7 \\
\textbf{LSP$^T$ (Ours)} & & 64.3 & 69.6 & 25.0 & 57.5 & 43.6 & 23.6 & \underline{47.3} & & 69.7 & 72.7 & 27.2 & 60.5 & 45.7 & 27.7 & \underline{50.6} \\
\textbf{LSP (Ours)}     & & 65.2 & 70.8 & 24.6 & 58.6 & 43.4 & 23.9 & \textbf{47.8}    & & 70.6 & 72.3 & 28.8 & 60.9 & 45.4 & 26.5 & \textbf{50.7} \\
\midrule
SVD-LLM (W) & \multirow{6}{*}{-70\%} & 28.7 & 53.4 & 14.8 & 49.3 & 26.5 & 20.3 & 32.2 & \multirow{6}{*}{-60\%} & 26.9 & 54.5 & 12.4 & 50.0 & 26.5 & 21.6 & 32.0 \\
Swift-SVD   & & 28.8 & 52.7 & 13.8 & 49.4 & 26.4 & 20.0 & 31.9 & & 30.2 & 56.1 & 12.6 & 49.3 & 27.3 & 22.2 & 33.0 \\
Dobi-SVD    & & 29.3 & 54.9 & 14.4 & 52.6 & 26.5 & 21.5 & 33.2 & & 28.1 & 55.5 & 14.6 & 52.0 & 27.1 & 21.9 & 33.2 \\
SVD-LLM     & & 39.5 & 58.2 & 16.2 & 50.9 & 29.9 & 21.3 & 36.0 & & 36.9 & 59.4 & 14.8 & 51.1 & 30.7 & 20.8 & 35.6 \\
\textbf{LSP$^T$ (Ours)} & & 50.3 & 63.2 & 20.0 & 52.9 & 33.8 & 22.0 & \underline{40.4} & & 55.5 & 68.7 & 21.6 & 53.7 & 37.6 & 22.9 & \underline{43.3} \\
\textbf{LSP (Ours)}     & & 55.6 & 66.1 & 21.8 & 51.7 & 36.1 & 21.9 & \textbf{42.2}    & & 56.9 & 68.8 & 21.2 & 55.3 & 37.3 & 23.3 & \textbf{43.8} \\
\bottomrule
\end{tabular}}
\vspace{-0.35cm}
\end{table}

\subsection{Vision Transformers}
\label{subsec:exp_vit}

Vision models let us control the calibration data cleanly: we can change the composition of the calibration set at a fixed size, and evaluate the compressed network both on its source task and, through linear-probes, on unseen datasets.
We use this to ask how much each method depends on seeing the evaluation distribution at compression time.
We build two calibration pools of the same size: a \emph{single} pool of 47k CIFAR-100 training images, in-domain with the model and its evaluation task, and a \emph{diverse} pool that keeps 10k of them and adds Food-101 \citep{bossard14}, CIFAR-10 \citep{krizhevsky2009learning}, EuroSAT \citep{helber2018introducing}, STL-10 \citep{coates2011stl10} and DTD \citep{cimpoi14describing}, capped at 10k images each.
The added datasets share no label space with CIFAR-100, so \Cref{tab:vit_results} compares only label-free methods: the training-free baselines; PELA, whose feature distillation needs no labels; and \method, whose distillation target is the dense model's output features for every image, whatever its source.

%%%%%%%%%
\begin{wraptable}{r}{0.44\linewidth}
\centering
\vspace{-0.45cm}
\caption{ViT-B/16 compressed on a \emph{single} pool ($\sim$47k CIFAR-100 images) or a \emph{diverse} pool ($\sim$47k images from CIFAR-100, Food-101, CIFAR-10, EuroSAT, STL-10 and DTD). \emph{Source}: CIFAR-100 accuracy; \emph{Transfer}: mean linear-probe accuracy on Pets, Aircraft and Places365, in neither pool; \emph{Gain}: diverse minus single transfer, computed before rounding.}
\vspace{-0.3cm}
\label{tab:vit_results}
\setlength{\tabcolsep}{2pt}
\resizebox{\linewidth}{!}{
\begin{tabular}{@{}cl@{\hspace{0.5em}}cc@{\hspace{0.7em}}ccc@{}}
\toprule
 & & \multicolumn{2}{c}{\textbf{Single pool}} & \multicolumn{3}{c}{\textbf{Diverse pool}} \\
\cmidrule(lr){3-4}\cmidrule(lr){5-7}
\textbf{Comp.} & \textbf{Method} & \textit{Source} & Transfer & \textit{Source} & \textit{Transfer} & \textit{Gain} \\
\midrule
\multicolumn{2}{l}{Dense (Base)} & 89.9 & 47.4 & 89.9 & 47.4 & --- \\
\midrule
\multirow{5}{*}{-30\%}
  & SliceGPT & 86.3 & 38.5 & 85.1 & 40.0 & +1.5 \\
  & SVD-LLM (W) & 88.2 & 41.2 & 87.7 & 42.3 & +1.1 \\
  & FLAR-SVD & 88.8 & 43.6 & 88.6 & 44.7 & +1.1 \\
  & PELA & 89.2 & 42.3 & 88.9 & 43.9 & +1.6 \\
  & \textbf{LSP (Ours)} & \textbf{89.4} & \textbf{43.9} & \textbf{89.2} & \textbf{47.5} & \textbf{+3.7} \\
\midrule
\multirow{5}{*}{-50\%}
  & SliceGPT & 83.3 & 33.8 & 78.8 & 36.9 & +3.2 \\
  & SVD-LLM (W) & 86.2 & 38.1 & 84.1 & 38.6 & +0.5 \\
  & FLAR-SVD & 86.4 & 38.1 & 83.1 & 36.2 & \textminus2.0 \\
  & PELA & \textbf{88.5} & \textbf{40.1} & \textbf{87.8} & 41.8 & +1.7 \\
  & \textbf{LSP (Ours)} & 88.3 & 37.6 & 87.7 & \textbf{44.7} & \textbf{+7.1} \\
\midrule
\multirow{5}{*}{-70\%}
  & SliceGPT & 66.3 & 26.6 & 44.5 & 29.7 & +3.1 \\
  & SVD-LLM (W) & 77.1 & 31.7 & 66.2 & 33.4 & +1.7 \\
  & FLAR-SVD & 71.5 & 26.6 & 64.6 & 27.9 & +1.3 \\
  & PELA & 85.3 & \textbf{34.1} & 82.8 & 38.3 & +4.2 \\
  & \textbf{LSP (Ours)} & \textbf{85.9} & 33.0 & \textbf{84.6} & \textbf{39.6} & \textbf{+6.6} \\
\bottomrule
\vspace{-1.1cm}
\end{tabular}}

\end{wraptable}
%%%%%%%%%
\myparagraph{Accuracy under calibration shift.}
Calibrated in-domain, the training-free baselines are close to \method at 30\%, but the gap opens with the ratio (\Cref{tab:vit_results}): at $-70\%$ \method still scores $85.9$ against the dense model's $89.9$, where the strongest training-free baseline falls to $77.1$.
PELA, which retrains every weight, is the only baseline on par with \method.
Moving the calibration data away from the evaluation distribution costs every method source accuracy, and \method the least; since the two pools are matched in size, this isolates the composition of the calibration data from its amount.

\myparagraph{Downstream transfer.}
\emph{Transfer} is the mean frozen-feature linear-probe accuracy over Pets \citep{parkhi12a}, Aircraft \citep{maji13fine-grained} and Places365 \citep{zhou2017places}, none of them in any calibration pool, and \emph{Gain} compares the two pools at matched size.
Calibration diversity helps the baselines little, while \method has the largest gain at every ratio and the best transfer on the diverse pool at all three ratios, including $-50\%$ and $-70\%$, where it trails on the single pool (per-target breakdown in \Cref{app:vit_transfer}).
Label-free training alone does not explain this: PELA trains on the same pools with the same budget, reaches the best single-pool transfer at $-50$ and $-70\%$, but gains less from diversity than \method at every ratio.
A label-free KL objective therefore not only gives most of the best results on both language and vision benchmarks, but also opens to a new axis for compression toward better zero-shot performance, exploiting unlabeled data to mimic the original abilities of the dense model.

\subsection{Inference Efficiency}
\label{subsec:compute}

\begin{table}[t]
\begin{minipage}{0.51\linewidth}
\centering
\vspace{0.2cm}
\includegraphics[width=\linewidth]{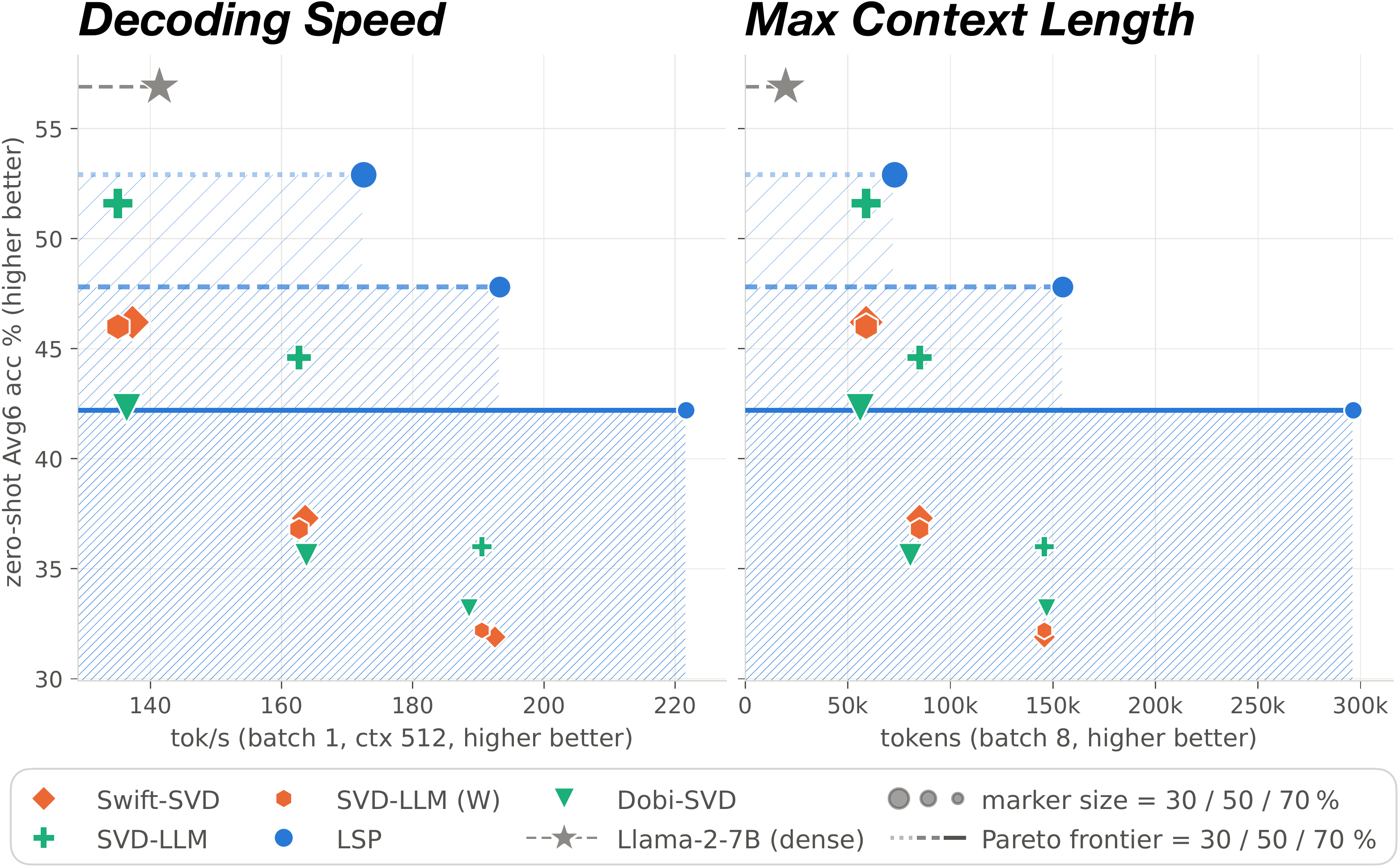}
\vspace{-0.5cm}
\caption{\textbf{Inference efficiency on Llama-2-7B.} Zero-shot accuracy (Avg over the six benchmarks of \Cref{tab:zeroshot_6bench}) at $-30/50/70\%$ compression (marker size) against CUDA-graph-compiled decode throughput at batch 1 and 512 tokens of context (left), and against the number of tokens per sequence whose weights and latent KV cache fit on one GH200 at batch size~8 (right). Lines trace the Pareto frontier at each ratio (dotted, dashed, solid), hatched in the color of the method that attains it.
}
\label{fig:inference}
\end{minipage}
\hfill
\begin{minipage}{0.47\linewidth}
\centering
\caption{\textbf{Inference efficiency of $-70\%$ compressed Llama-2-7B.} The factorized methods match on weights and FLOPs; \method's tied low-rank projector caches one narrow latent per layer instead of separate K and V, so its KV cache is $2.2$--$2.6\times$ smaller than the untied factorizations' and $14.7\times$ smaller than dense, it holds the longest context on one GPU; \method also decodes fastest at full-width KV cache.
}
\vspace{-0.2cm}
\label{tab:inference}
\footnotesize
\setlength{\tabcolsep}{2pt}
\resizebox{\linewidth}{!}{%
\begin{tabular}{@{}lccccc@{}}
\toprule
& \textit{Dense} & \textit{SVD-} & \textit{Swift-} & \textit{Dobi-} & \textbf{LSP} \\
& & \textit{LLM} & \textit{SVD} & \textit{SVD} & \textbf{(Ours)} \\
\midrule
\multicolumn{6}{@{}l}{\textit{Memory (batch 8)}} \vspace{0.05cm}\\
\textbf{Weights (GiB)} & 12.55 & 4.11 & 4.11 & 4.12 & 4.11 \\
\textbf{KV floats/tok/layer} & 8192 & 1228 & 1228 & 1434 & \textbf{556} \\
\quad \textbf{vs.\ dense} & -- & 6.7$\times$ & 6.7$\times$ & 5.7$\times$ & \textbf{14.7}$\times$ \\
\textbf{W+KV @128k (GiB)} & 524.6 & 80.9 & 80.9 & 93.7 & \textbf{38.9} \\
\quad \textbf{vs.\ dense} & -- & 6.5$\times$ & 6.5$\times$ & 5.6$\times$ & \textbf{13.5}$\times$ \\
\textbf{Max ctx (k tokens)} & 19.7 & 145.8 & 145.8 & 124.9 & \textbf{322.1} \\
\midrule
\multicolumn{6}{@{}l}{\textit{Compute and speed (batch 1)}} \vspace{0.05cm}\\
\textbf{Prefill TFLOPs @2k} & 29.26 & 10.69 & 10.69 & 10.71 & 10.69 \\
\textbf{Decode GFLOP/step @2k} & 14.29 & 5.22 & 5.22 & 5.23 & 5.22 \\
\textbf{Decode tok/s @512} & 141 & 191 & 193 & 189 & \textbf{221} \\
\quad \textbf{@4k} & 85 & 100 & 101 & 99 & \textbf{108} \\
\quad \textbf{@16k} & 36 & 38 & 38 & 38 & \textbf{40} \\
\bottomrule
\end{tabular}
}
\end{minipage}
\vspace{-0.3cm}
\end{table}

\Cref{fig:inference} compares decode throughput and the longest context that fits on one GPU for Llama-2-7B at all three ratios (\method$^T$ has the same factor shapes and is omitted).
\Cref{tab:inference} details the $-70\%$ checkpoints in two deployments of the same factorized model: with a standard full-width KV cache, in which throughput is measured (SDPA, CUDA-graph decoding), and with a latent cache, for which cache size and context capacity are computed (\Cref{app:inference_models}).

\myparagraph{Decode throughput.}
At small batch sizes, decoding is memory-bound: each step reads every weight once.
At the same compression ratio, all factorizations read about the same number of weight bytes, so what separates them is how these bytes are split into matrix multiplies, since small multiplies achieve lower memory bandwidth than large ones.
\method issues fewer of them: a tied group runs one shared input factor for Q/K/V where an untied factorization runs three, and a unit the allocator leaves dense (\Cref{subsec:analysis}) runs as one matrix multiply rather than two.
It is the fastest factorized model in all model--ratio settings we measure (\Cref{app:inference_models}); on Llama-2-7B, it decodes faster than the dense model at every ratio ($1.21\times$, $1.36\times$, and $1.56\times$ at $-30$, $-50$, and $-70\%$), whereas every untied factorization is slower than dense at $-30\%$.
\smallskip\\
\myparagraph{Latent-cache memory.}
Factorizing $\mW = \mB\mA$ (\Cref{eq:merge}) does not by itself shrink the KV cache, since $\mB$ restores the full width, but caching the latent $\vz = \mA\vx$ does, given a dedicated attention implementation (\Cref{app:merge}).
As in MLA and Palu \citep{deepseekai2024deepseekv2, chang2024palucompressingkvcachelowrank}, the value factor absorbs into the output projection, while under rotary embeddings keys are rebuilt from the latent at each step, at a cost linear in context length and latent width that any latent factorization pays; the latent cache thus serves capacity rather than speed.
A tied group caches one latent where an untied factorization caches two, and the allocator compresses Q/K/V strongly (\Cref{subsec:analysis}), so \method has the narrowest cache (\Cref{tab:inference}).
At $-70\%$, with $128$k cached tokens and batch size $8$, weights plus cache take $13.5\times$ less memory than dense, against at most $6.5\times$ for untied factorizations.
A 95.5\,GiB GPU thus holds about $320$k cached tokens per sequence, against at most $146$k: a memory capacity well beyond Llama-2-7B's 4k training context.

\subsection{What Does \method Compress?}
\label{subsec:analysis}

\myparagraph{Allocation.}
On Llama-2-7B measured-KL allocation is far from uniform (\Cref{fig:drop_diff}, left).
The Q/K/V groups are compressed far more than the MLP projections, even though each direction removed from gate/up saves more parameters ($26{,}112$) than one removed from Q/K/V ($16{,}384$): per saved parameter, attention inputs are markedly cheaper to compress.
This is what keeps the shared K/V latent narrow at high ratios (about a seventh of full rank at $-70\%$), and hence drives the KV-cache savings of \Cref{subsec:compute}.
Later blocks retain more rank than earlier ones at every ratio, and at low compression many units (untied layers or tied groups) stay dense, $53$ of $128$ at $-30\%$: a unit is factorized only when its allocated rank falls below the break-even rank at which the two factors become smaller than the dense weight.
The same contrast between attention and MLP projections holds across all LLMs (\Cref{fig:rank_profile_models}).
\smallskip\\
\myparagraph{Retained subspace.}
Let $\vw_i$ be the $i$-th singular vector of an original weight on its projected side, and $a_i = \|\mP\vw_i\|_2 \in [0,1]$.
Projection shrinks the $i$-th rank-one term of the weight from $\sigma_i$ to $\sigma_i a_i$, so $\Delta_i = \sigma_i(a_i - 1)$ is the amplitude lost along that direction (\Cref{fig:drop_diff}, middle).
Unlike weight-SVD truncation, removal has no sharp cutoff: it spans the whole spectrum, deepest on the leading directions in Q/K/V and inside the spectrum in the other projections.
Learning shifts it further (\Cref{fig:drop_diff}, right): 
relative to NoLSP, which shares the allocated ranks, \method generally removes more of its high singular directions (left part) in exchange for small singular direction ones (right part).
The leading direction changes most: direction~1 has the largest absolute difference to NoLSP. Notice that learning, the only difference between the two, lowers perplexity at $-70\%$ from $222.8$ to $10.9$ (\Cref{tab:results}).
The weight spectrum thus ranks directions poorly: whitened truncation discards trailing directions the output depends on, and learning corrects it, most strongly along the leading direction.

\begin{figure}[t]
\centering
\vspace{-0.15cm}
\includegraphics[width=0.33\linewidth]{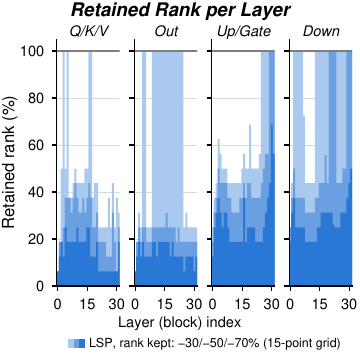}\hfill
\includegraphics[width=0.33\linewidth]{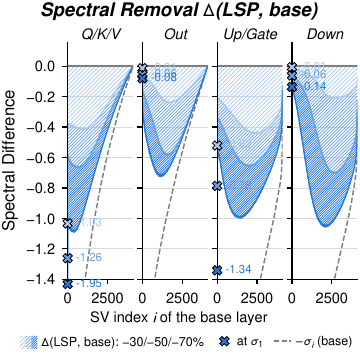}\hfill
\includegraphics[width=0.33\linewidth]{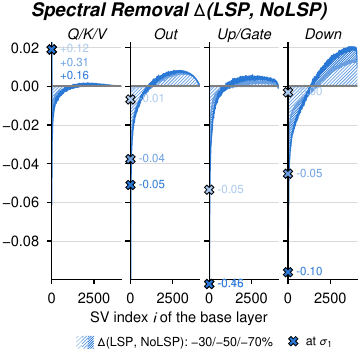}
\vspace{-0.2cm}
\caption{
\textbf{Rank allocation and spectral removal on Llama-2-7B.}
\emph{Left:} fraction of full rank retained by each unit after compression, by block and projection type; measured-KL allocation is far from uniform.
\emph{Middle:} amplitude $\Delta_i$ lost along each singular direction of the original weights; average over layers;
\emph{Right:} difference to NoLSP, $\sigma_i(a_i^{\text{LSP}} \!-\! a_i^{\text{NoLSP}})$; the leading direction changes most, and outside Q/K/V \method removes more of it and retains more of the trailing directions, most visibly in the down projection.
All panels show runs at $-30/50/70\%$ compression (light to dark). Crosses give the delta for $\sigma_1$. Plots for more models and layers in \Cref{app:subspace}.
}
\vspace{-0.3cm}
\label{fig:drop_diff}
\end{figure}

\section{Related Work}
\label{sec:related}

\myparagraph{Pruning.} Pruning removes weights scored by magnitude \citep{hanMagnitude, hanDeepCompression}, by their movement during fine-tuning \citep{movementPruning}, or by a local quadratic model of the loss \citep{OBdamage, OBsurgeon}.
At LLM scale, the scores come from layer-wise reconstruction \citep{sparsegpt}, activation-scaled magnitudes \citep{sun2024wanda}, gradient-based group importance \citep{ma2023llmpruner} or layer redundancy \citep{men2025shortgpt}.
Recent methods learn what to remove against a global objective with the weights frozen, as semi-structured masks \citep{fang2024maskllm, liu2025proxsparse} or per-block widths \citep{gao2024displlm}, part of a broader shift from local to global criteria in structured pruning \citep{wang2026localglobal}.
\method brings this idea to low-rank compression: the removed set is a continuous subspace rather than a set of coordinates, so no discrete relaxation is needed, and the savings are realized as dense low-rank factors.
\smallskip\\
\myparagraph{Low-rank compression.} Trained weights are not low-rank themselves, so most low-rank methods rely on the activations lying in lower-dimensional subspaces \citep{yuan2024asvdactivationawaresingularvalue, garg2024revealing}.
They choose the removed subspace per layer in closed form, either from activation statistics, as in ASVD \citep{yuan2024asvdactivationawaresingularvalue}, SliceGPT \citep{ashkboos2024slicegpt}, MoDeGPT \citep{lin2025modegptmodulardecompositionlarge}, FLAR-SVD \citep{Thoma_2025_CVPR} and further variants \citep{huang2025sola, li2025moesvd, chiang2026uniql}, or from the layer's output reconstruction error, as in SVD-LLM \citep{wang2025svdllm} and Swift-SVD \citep{qi2026swiftsvd}.
Acting on one layer in isolation, these criteria do not see the network output, and their errors accumulate through depth \citep{odema2026understanding}.
Where the loss is used, it enters through a local surrogate, as in FW-SVD \citep{hsu2022language} and LLM-Surgeon \citep{van2023llm}, or by fine-tuning the weights after truncation, as in SVD-LLM's LoRA stage and PELA \citep{guo2024pela}; \method instead optimizes the removed subspace itself against the network output, with the weights frozen.
\smallskip\\
\myparagraph{Rank allocation and shared bases.} Beyond choosing directions, compression also decides how the remaining capacity is laid out across the network.
Ranks are allocated non-uniformly across matrices from measured or estimated sensitivity \citep{yuan2024asvdactivationawaresingularvalue, wang2025svdllmv2, qi2026swiftsvd} or learned through a relaxation of the truncation position \citep{qinsi2025dobisvd}, and factors are shared among matrices, across layers \citep{wang2025basis} or among the Q/K/V projections that read the same activation \citep{wang2025qsvd}.
The latter two adapt the structure of the compressed model to the network, but the basis at each rank remains closed-form.
\method uses the same two levers: projectors are tied within each group of layers that share an input, and ranks are allocated per compression unit by applying whitened projections and measuring the change in output KL (\Cref{subsec:alloc_init}).

\section{Conclusion}
\label{sec:conclusion}

We introduced \methodlong (\method), which compresses a pretrained network by learning, for each linear layer, the subspace the network can do without.
Its premise is that statistical redundancy, whether measured by activation energy or by layer-wise reconstruction error, is not functional redundancy.
\method therefore selects the removed subspaces against the network output: it allocates ranks by measured output KL and optimizes all projectors jointly, with the pretrained weights frozen.
Across four decoder LLMs, one of its two objectives attains the lowest perplexity in all twelve model--ratio settings, and both achieve the best zero-shot accuracy among low-rank methods.
The margin grows with the compression ratio: at $-70\%$ on Llama-2-7B, \method reaches $10.9$ perplexity, against $13.3$ for the strongest baseline and $222.8$ without learning.
On a vision transformer, it is the most robust to calibration shift and transfers best when calibrated on diverse unlabeled data.
The merged model is also the fastest factorization we measure, decoding up to $1.56\times$ faster than the dense model.
Its shared latent reduces the memory of weights and KV cache by $13.5\times$ at a 128k-token context, twice the savings of untied factorizations.    

\section*{Acknowledgments}
Our work was partially funded by the ERC (853489 - DEXIM) and the Alfried Krupp von Bohlen und Halbach Foundation, which we thank for their support.
The authors gratefully acknowledge the Gauss Centre for Supercomputing e.V. (www.gausscentre.eu) for funding this project by providing computing time on the Supercomputer JUPITER at Jülich Supercomputing Centre (JSC). This work was supported by the EuroHPC JU and the Gauss Centre for Supercomputing (GCS) through funding by the European Commission, German Federal Ministry of Research, Technology and Space (BMFTR), the Ministry of Culture and Science of the State of North Rhine-Westphalia (MKW). This work also benefited from Hi! PARIS and state funding managed by the French National Research Agency (ANR) under the France 2030 program, reference ANR-23-IACL-0005.

\section*{AI Disclosure Statement}

In this work, we used generative AI tools for: design or provide feedback on research  methodology or experiments (feedback to double-check for correctness); implement methods (implementation portions; tested and manually checked); assist in formulating mathematical claims, providing critical ingredients, and helping writing proofs (with each step checked by two authors), refine hypotheses, support qualitative and thematic data analysis (data visualization and organization, tested and manually checked), interpret results (summarization).
We have not used generative AI tools for: help develop theoretical models or conceptual frameworks, assist with translation, clean and reformat dataset, generating synthetic data.
We take responsibility for the final content of this work, including text, claims or artifacts produced with the aid of generative AI.

\section*{Ethics Statement}
By substantially lowering the parameter and compute footprint of large models, LSP contributes to the democratization of pretrained transformers, making strong language and vision models deployable on edge devices and in domain-specialized applications where retraining from scratch is infeasible. As with any compression method, the compressed model is a lossy approximation of the original; while our experiments show small aggregate degradation, this loss may be unevenly distributed across subpopulations, domains, or rare-event behaviors that the calibration distribution under-represents. We recommend that downstream practitioners evaluate compressed models on the same fairness, robustness, and safety metrics they would apply to the dense baseline before deploying them.

\bibliography{bibliography}
\bibliographystyle{iclr2027_conference}

\newpage

\appendix
\section{Method Details}
\label{app:technical_details}

\subsection{Optimality of the Orthogonal Projector}
\label{app:proj_optimality}
\Cref{sec:lsp} applies the projector $\mP = \mI - \mU\mU^\top$. The following proposition is why that particular operator: once the removed subspace is fixed, the orthogonal projector is the unique least-disruptive map that annihilates it.

\textbf{Proposition.}
Let $S \subset \mathbb{R}^d$ have orthonormal basis
$\mU \in \mathbb{R}^{d \times k}$, and let $\mP = \mI - \mU\mU^\top$ denote
the orthogonal projector onto $S^\perp$. Among all
$\mQ \in \mathbb{R}^{d \times d}$ with $\mQ\vu = 0$ for every $\vu \in S$,
$\mP$ is the unique minimizer of $\|\mI - \mQ\|_F$, attaining the optimal
value $\sqrt{k}$.

\textbf{Proof.}
Write $\mQ = \mI - \mE$ so the annihilation constraint $\mQ\mU = 0$ becomes
$\mE\mU = \mU$. Let $\mU_\perp \in \mathbb{R}^{d \times (d-k)}$ complete
$\mU$ to an orthonormal basis of $\mathbb{R}^d$. In this basis, every
admissible $\mE$ has the block form
\begin{equation}
    \mE \;=\;
    \begin{bmatrix} \mU & \mU_\perp \end{bmatrix}
    \begin{bmatrix} \mI_k & \mC \\ 0 & \mD \end{bmatrix}
    \begin{bmatrix} \mU & \mU_\perp \end{bmatrix}^\top,
    \qquad
    \mC \in \mathbb{R}^{k \times (d-k)},\;
    \mD \in \mathbb{R}^{(d-k) \times (d-k)},
    \label{eq:E_block_form}
\end{equation}
where the constraint $\mE\mU = \mU$ has fixed the $(1,1)$ block to $\mI_k$
and the $(2,1)$ block to $0$, while $\mC$ and $\mD$ are free. Orthogonal
invariance of the Frobenius norm and the block decomposition give
\begin{equation}
    \|\mE\|_F^2
    \;=\; \|\mI_k\|_F^2 + \|\mC\|_F^2 + \|\mD\|_F^2
    \;=\; k + \|\mC\|_F^2 + \|\mD\|_F^2 \;\ge\; k,
\end{equation}
with equality iff $\mC = 0$ and $\mD = 0$. In that case
$\mE = \mU\mU^\top$ and $\mQ = \mI - \mU\mU^\top = \mP$, establishing
optimality and uniqueness.
\hfill $\square$

The proposition settles the operator, not the subspace: which directions are removed is chosen by the whitened criterion of \Cref{subsec:alloc_init} at initialization and by the loss thereafter, whereas among all maps that annihilate a given $S$ only $\mP$ leaves $S^\perp$ untouched. Because $\mP$ is a projector ($\mP^2 = \mP$), the modulated operator of \Cref{sec:lsp} at $\vm = \mathbf{1}$ is the convex combination $(1-\alpha)\mI + \alpha\mP$, a true interpolation between identity and projection rather than an arbitrary contraction.

\subsection{Practical Design Choices}
\label{app:design_choices}
This section details the tied groups and the three design choices of \Cref{sec:lsp}.
Throughout, $\ell$ indexes projectors (one per untied layer or tied group), $d_\ell$ is the dimension of the projected side, and $k_\ell$ is the number of removed directions.

\myparagraph{Activation-space form.}
Applying the projector through the activation keeps no dense $d_\text{out}\times d_\text{in}$ projected weight in the autograd graph: the only additional state is $\mU_\ell$ and the narrow activation $\mU_\ell^\top\vx \in \mathbb{R}^{k_\ell}$ per token, on top of the activations the frozen network already stores.
The fused QR backward retains $\mathcal{O}(d_\ell k_\ell + k_\ell^2)$ state, compared with $\mathcal{O}(d_\ell k_\ell^2)$ for an explicit Gram--Schmidt loop; both return an orthonormal basis of $\operatorname{span}(\mV_\ell)$ when $\mV_\ell$ has full column rank, and hence the same projector in exact arithmetic.
Since $k_\ell \le d_\ell$, the projector and QR state totals $\mathcal{O}(\sum_\ell d_\ell k_\ell)$ over the network, plus $\mathcal{O}(\sum_\ell k_\ell)$ per token for the narrow activations.

\myparagraph{Tied groups.}
For an input-tied group $\mathcal{G}$, we apply \Cref{eq:lsp_proj} to the row-stacked weight $[\mW^{1}; \dots; \mW^{|\mathcal{G}|}] \in \mathbb{R}^{(\sum_{m} d^m_\text{out}) \times d_\text{in}}$.
At a common rank $r$, the merged group costs $r\,(d_\text{in} + \sum_{m \in \mathcal{G}} d^m_\text{out})$ parameters, against $r\sum_{m \in \mathcal{G}} (d_\text{in} + d^m_\text{out})$ for $|\mathcal{G}|$ separate factorizations (\Cref{eq:merge}).
Sharing thus saves $(|\mathcal{G}|-1)\,d_\text{in} r$ parameters, which the group can spend on a higher rank, at the cost of constraining all members to the same input subspace.
The same construction applies to an output-tied group, whose members write the same output: the weights are column-stacked, the output factor is shared, each member keeps its own input factor, and the merged group costs $r\,(d_\text{out}+\sum_{m\in\mathcal{G}}d^m_\text{in})$ parameters (\Cref{eq:merge_output}).

\myparagraph{Warm-up, direction dropout, and orthogonality penalty.}
A projector that removes $k$ directions is never the identity: $\|\mI - \mP\|_F = \sqrt{k}$ for every $\mU$, an analogue of the fixed distance of reflections from the identity noted by \citet{bini2024ether}.
The ramp of $\alpha$ in $\mP_{\alpha,\vm}$ (\Cref{sec:lsp}) therefore starts training from the identity.
The mask entries $m_i \sim \mathrm{Bernoulli}(1-p)$ are redrawn at every step, independently for each member of a tied group over the shared $\mU$.
We apply no $1/(1-p)$ rescaling: it would give each removed direction the coefficient $1 - \alpha/(1-p)$, which becomes negative as $\alpha \to 1$ and would reflect the direction rather than remove it.
Without rescaling, at $\alpha = 1$ the modulated $\mP$ is an orthogonal projector onto the complement of the directions removed at that step.
Both $\alpha < 1$ and the mask relax the target rank during training; it is reached only at $\alpha = 1$ and $\vm = \mathbf{1}$.

The orthogonality penalty acts on the unconstrained columns of $\mV_\ell$:
\begin{equation}
    \mathcal{L}_\text{ort} = \sum_{\ell:\,k_\ell > 1} \frac{1}{k_\ell(k_\ell-1)} \sum_{i \neq j} \big| \langle \vv^{(\ell)}_i, \vv^{(\ell)}_j \rangle \big|.
    \label{eq:orth_reg}
\end{equation}
It discourages correlated columns, which in practice stabilizes QR and its gradient; orthonormality of $\mU_\ell$ itself is already enforced by QR.
We use the raw inner product rather than a cosine: the columns are unit-norm at initialization, and since the optimizer applies no weight decay, the raw form also keeps their norms from growing.
Unlike $\mP$ itself, direction dropout and the penalty depend on the chosen basis of $\operatorname{span}(\mV_\ell)$, so during training the optimization is not fully invariant to rotations within the subspace; the exported projector is.

\subsection{Whitened Initialization and Its Excess Error}
\label{app:whitened_excess}
We prove \Cref{prop:whitened}.
Drop the superscript $w$ from the SVD of $\mW\mS$ and complete the singular bases to $[\mU_r\;\mU_{>r}]$ and $[\mV_r\;\mV_{>r}]$, padding $\boldsymbol{\Sigma}_{>r}$ with zero singular values where needed, so that $\|\boldsymbol{\Sigma}_{>r}\|_F^2 = \sum_{i>r}\sigma_i^2$.
As in the proposition, $\mC_r=\mS\mV_r$, $\mC_{>r}=\mS\mV_{>r}$, $\mathcal{K}=\operatorname{span}(\mC_r)$ and $\mP_\text{init}=\mC_r\mC_r^+$.
Since $[\mV_r\;\mV_{>r}]$ is orthogonal, rotating by it preserves the Frobenius norm, so for any $\mZ$
\begin{equation}
    \mathcal{E}(\mZ) = \|(\mW-\mZ)\mC_r\|_F^2 + \|(\mW-\mZ)\mC_{>r}\|_F^2,
    \qquad
    \mW\mC_r=\mU_r\boldsymbol{\Sigma}_r,\quad
    \mW\mC_{>r}=\mU_{>r}\boldsymbol{\Sigma}_{>r}.
    \label{eq:whit_split}
\end{equation}

\myparagraph{(i) Output side.}
$(\mI-\mU_{>r}\mU_{>r}^\top)\mW = \mU_r\mU_r^\top\mW$, the first form of \Cref{eq:whitening}, so its error is the optimum $\|\boldsymbol{\Sigma}_{>r}\|_F^2$.

\myparagraph{(ii) Input side.}
Since $\mC_r^\top\mS^{-\top}\mV_{>r}=\mV_r^\top\mV_{>r}=0$ and the dimensions add to $d_\text{in}$, removing $\operatorname{span}(\mS^{-\top}\mV_{>r})$ leaves exactly $\mathcal{K}$, so the initialized weight is $\mW\mP_\text{init}$.
In \Cref{eq:whit_split}, the error on $\mC_r$ vanishes because $\mP_\text{init}\mC_r=\mC_r$, and the error on $\mC_{>r}$ is
\begin{equation}
    \mW(\mI-\mP_\text{init})\mC_{>r}
    =\mU_{>r}\boldsymbol{\Sigma}_{>r}
     -\mU_r\boldsymbol{\Sigma}_r\mC_r^+\mC_{>r}.
\end{equation}
The two terms have orthogonal column spaces, giving
\begin{equation}
    \mathcal{E}(\mW\mP_\text{init})
    =\|\boldsymbol{\Sigma}_{>r}\|_F^2
     +\|\boldsymbol{\Sigma}_r\mC_r^+\mC_{>r}\|_F^2,
    \label{eq:whit_excess_proof}
\end{equation}
where the first term is the optimal rank-$r$ error and the second the excess.
Because $\mS$ is invertible, $\mC_r$ has full column rank and $\mC_r^+=(\mC_r^\top\mC_r)^{-1}\mC_r^\top$; if $\sigma_r>0$, $\boldsymbol{\Sigma}_r$ is invertible as well, so the excess vanishes if and only if $\mC_r^\top\mC_{>r}=0$.
$\mC_r^+\mC_{>r}$ holds the least-squares coefficients of the removed input directions on the kept ones, so the excess is their squared norm weighted by the kept singular values.

\myparagraph{(iii) Tied groups.}
For $|\mathcal{G}|$ layers sharing their input, and hence $\mS$, such as Q/K/V, one shared kept subspace constrains the row-stacked approximation of $\bar\mW=[\mW_1;\dots;\mW_{|\mathcal{G}|}]$ to rank $r$, and the summed error is $\sum_m\|(\mW_m-\widehat{\mW}_m)\mS\|_F^2=\|(\bar\mW-\widehat{\bar\mW})\mS\|_F^2$.
By Eckart--Young \citep{eckart1936approximation}, the leading right singular subspace of $\bar\mW\mS$ gives the optimal shared subspace in whitened coordinates, and (ii) applies to $\bar\mW$.
For $|\mathcal{G}|$ layers sharing one output subspace, such as K and V under grouped-query attention, with input factors $\mS_m$, a shared output projector $\mP_\text{out}$ gives the summed error $\sum_m\|(\mI-\mP_\text{out})\mW_m\mS_m\|_F^2=\|(\mI-\mP_\text{out})[\mW_1\mS_1\;\cdots\;\mW_{|\mathcal{G}|}\mS_{|\mathcal{G}|}]\|_F^2$, which the leading left singular subspace of the column-stacked whitened weight minimizes, and (i) applies.
Output tying shares output coordinates rather than a projected activation, so on its own it yields no single latent to cache.
\hfill$\square$

\myparagraph{Alternative recast.}
One could instead remove the whitened truncation's null space $\operatorname{span}(\mC_{>r})$ itself.
With $\mP_{>r}=\mC_{>r}\mC_{>r}^+$, the same split gives
\begin{equation}
    \mathcal{E}\big(\mW(\mI-\mP_{>r})\big)
    =\|\boldsymbol{\Sigma}_{>r}\|_F^2
     +\|\boldsymbol{\Sigma}_{>r}\mC_{>r}^+\mC_r\|_F^2 .
\end{equation}
Neither excess dominates in general: the alternative's is weighted by the smaller singular values $\boldsymbol{\Sigma}_{>r}$, but involves $\mC_{>r}^+$, which can amplify error when the removed directions are ill-conditioned.
We use $\mP_\text{init}$, which reproduces the whitened truncation exactly on the kept subspace.

\myparagraph{Numerical stabilization and output-side bases.}
The implementation accumulates the unnormalized input Gram matrix $\mH=n\mG$.
Cholesky is first attempted without a ridge; if it fails, the implementation replaces $\mH$ by $\mH+\eta\mI$, with $\eta=-\lambda_\text{min}(\mH)+10^{-6}$, and retries.
Scaling the Cholesky factor by $1/\sqrt{n}$ gives a factor of $\mG+(\eta/n)\mI$, so the whitened optimum and the excess identities then concern the regularized objective
\begin{equation}
    \mathcal{E}_\eta(\mZ)=\mathcal{E}(\mZ)+(\eta/n)\|\mW-\mZ\|_F^2,
\end{equation}
with $\eta=0$ when the first factorization succeeds.
For an output-tied group, the implementation instead decomposes the summed bias-free output Gram matrix, proportional to $\sum_m\mW_m\mG_m\mW_m^\top$; its eigenvectors are the left singular vectors of $[\mW_1\mS_1\;\cdots\;\mW_{|\mathcal{G}|}\mS_{|\mathcal{G}|}]$, so no input Cholesky factor is required.
The added ridge, $0.01$ times the mean diagonal, shifts eigenvalues without changing eigenspaces in exact arithmetic.

\subsection{Measured-KL Allocation}
\label{app:alloc_details}
This section states the allocator of \Cref{subsec:alloc_init} in full.
For a unit $u$ with projected dimension $d_u$, dense parameter count $N_u$, and factor count $c_u(r)$ at retained rank $r$ (input- or output-tied costs of \Cref{app:design_choices}, counting each shared factor once), removing the trailing $k$ basis directions saves
\begin{equation}
    s_u(k)=\begin{cases}
        0, & k=0 \text{ (dense)},\\
        N_u-c_u(d_u-k), & k>0,
    \end{cases}
    \label{eq:unit_savings}
\end{equation}
which is negative for small $k$, since a shallow truncation stores more parameters than the dense unit.
Each measurement applies the actual initialized orthogonal projector, including the input-side recast of \Cref{prop:whitened}, and $\Delta_u(0)=0$.
We evaluate $\Delta_u$ on a grid of removal fractions, replace each curve of \Cref{eq:kl_cost} by its running maximum so that it is nondecreasing, and interpolate linearly between grid counts, yielding $\widetilde\Delta_u$.
For a target compression ratio $\tau$, allocation uses the separable surrogate
\begin{equation}
    \min_{\{k_u\}}\;\sum_u\widetilde\Delta_u(k_u)
    \quad\text{subject to}\quad\sum_u s_u(k_u)\geq \tau\sum_u N_u,
    \qquad k_u\in\mathcal{K}_u,
    \label{eq:alloc_objective}
\end{equation}
where $\mathcal{K}_u$ contains zero and the grid removal counts.
Starting with all units dense, we solve it greedily: among all units and all grid counts $k'>k$ with $s_u(k')>s_u(k)$, we take the move with the smallest marginal cost
\begin{equation}
    \frac{\widetilde\Delta_u(k')-\widetilde\Delta_u(k)}{s_u(k')-s_u(k)},
    \label{eq:marginal_cost}
\end{equation}
until the target is reached.
Because a move may span several grid points, a unit's first move can jump past the break-even rank; sensitive units may never move and stay dense.
An integer search finally trims the last move to the closest attainable target.

\newpage
\subsection{\method Algorithm}
\label{app:alg_pseudocode}
\Cref{alg:lsp_oneshot} states the pipeline of \Cref{sec:lsp,subsec:alloc_init,subsec:training_loop} end to end.

\begin{algorithm}[H]
\caption{\method.}
\label{alg:lsp_oneshot}
\begin{algorithmic}[1]
\Require model $M$; calibration set $\mathcal{D}_\text{cal}$; held-out validation set $\mathcal{D}_\text{val}$; target compression $\tau$; epochs $E$; modulation epochs $E_\alpha$; dropout $p$; orthogonality weight $\lambda_\text{ort}$; learning rate $\eta$; patience $P$; merge tolerance $\epsilon_\text{SVD}$
\State Replace each targeted linear layer (all but the head and the embeddings) with its \method counterpart; tie compatible input groups (Q/K/V, gate/up), or K/V on the output side under grouped-query routing
\State \textbf{Gather} on $\mathcal{D}_\text{cal}$ without gradients: collect the input Gram for each input-side unit and each member's bias-free output Gram for an output-side unit
\State \textbf{Build bases}: use the row-stacked whitened weight for input-tied groups, or the summed output Gram for output-tied groups (\Cref{app:whitened_excess})
\State \textbf{Measure}: cache dense outputs on $\mathcal{D}_\text{alloc}\subseteq\mathcal{D}_\text{cal}$; truncate one unit at a time over the removal grid and record $\Delta_u(k)$ (\Cref{eq:kl_cost})
\State \textbf{Allocate}: take running maxima and interpolate; compare all higher grid endpoints by marginal KL per parameter saved (\Cref{eq:marginal_cost}); integer-trim only the last selected interval to the nearest target
\State Measure the KL of the joint allocation and compare it with $\sum_u\widetilde\Delta_u(k_u)$; fix the selected ranks for training
\State \textbf{Initialize}: select each unit's trailing $k_u$ directions; set $\mV\leftarrow\operatorname{qf}(\mS^{-\top}\mV^w_{>r})$ for an input-side unit, or use the trailing output-Gram eigenvectors for an output-side unit \Comment{NoLSP at the selected ranks and ties}
\State Set up Adam over $\{\mV_\ell\}$ with a linear-warmup cosine schedule over $E$ epochs; every pretrained weight and bias stays frozen
\For{epoch $e = 1, \ldots, E$ \textbf{or} until early stop}
    \For{micro-batch in $\mathcal{D}_\text{cal}$}
        \State $\alpha\leftarrow\min\!\big(1,\ t/(E_\alpha N_\mu)\big)$ \Comment{$t$ = cumulative micro-batch index, starting at 1; $N_\mu$ = micro-batches per epoch per replica}
        \For{compressed unit}
            \State $\mU \leftarrow \operatorname{qf}(\mV)$; draw $\vm \sim \mathrm{Bernoulli}(1-p)^{k}$, independently per member of a tied group
            \State Forward with $\mP_{\alpha,\vm} = \mI - \alpha\,\mU\operatorname{diag}(\vm)\mU^\top$, applied in activation space
        \EndFor
        \If{distillation}
            \State $p_\text{dense} \leftarrow$ forward of the same $M$ with the projections disabled, without gradients \Comment{no second copy of the weights}
        \EndIf
        \State $\mathcal{L} \leftarrow \mathcal{L}_\text{obj} + \lambda_\text{ort}\mathcal{L}_\text{ort}$, \quad $\mathcal{L}_\text{obj} \in \{\mathcal{L}_\text{KL}\ (\method),\ \mathcal{L}_\text{task}\ (\method^T)\}$
        \State Backprop; update $\{\mV_\ell\}$ and step the schedule after each gradient-accumulation window
    \EndFor
    \State Evaluate on $\mathcal{D}_\text{val}$ (validation perplexity for language models, accuracy for the ViT); record $\{\mV_\ell\}$ if best so far; \textbf{break} if no improvement for $P$ epochs
\EndFor
\State Restore the best-validation $\{\mV_\ell\}$
\State \textbf{Merge} at $\alpha=1$, $\vm=\mathbf{1}$ (\Cref{app:merge}): factor the row-stacked input-projected or column-stacked output-projected weights by thin SVD, sharing $\mA$ or $\mB$ respectively; drop singular values below $\epsilon_\text{SVD}$
\State \Return permanently compressed model $M'$
\end{algorithmic}
\end{algorithm}

\subsection{Merge Details}
\label{app:merge}
An input-tied group factors as $\mW_m\mP=\mB_m\mA$ with $\mB_m=\mW_m\mU_\perp$ and shared $\mA=\mU_\perp^\top$, $\mU_\perp \in \mathbb{R}^{d_\text{in} \times r}$ an orthonormal basis of $\operatorname{span}(\mU)^\perp$ in the input space and $r = d_\text{in} - k_\mathcal{G}$ (\Cref{eq:merge}). An output-tied group factors as
\begin{equation}
    \mP\mW_m \;=\; \mB\mA_m,
    \qquad
    \mB \;=\; \mU_\perp \in \mathbb{R}^{d_\text{out} \times r},
    \quad
    \mA_m \;=\; \mU_\perp^\top\mW_m \in \mathbb{R}^{r \times d^m_\text{in}},
    \label{eq:merge_output}
\end{equation}
with $\mU_\perp$ an orthonormal basis of $\operatorname{span}(\mU)^\perp$ in the output space and $r = d_\text{out} - k_\mathcal{G}$. In both cases the shared factor is the bare basis and the per-member factor carries the weight. Both factorizations are exact before numerical truncation and leave the bias unchanged.
In practice, factors are read from a thin SVD of the merged weight, row-stacked for an input tie and column-stacked for an output tie. If all $r$ latent directions are nonzero, this differs from the displayed factors only by an invertible change of latent basis. Rank deficiency or the numerical merge tolerance can reduce the exported rank below $r$; dropping nonzero singular values at that tolerance introduces a further approximation.

For input-tied K/V, a latent cache stores $\vz=\mA\vx$ and reconstructs keys and values as $\mB_k\vz$ and $\mB_v\vz$, adding biases and applying rotary embeddings after reconstruction. Output-tied K/V share $\mB$ but generally need separate latents $\mA_k\vx$ and $\mA_v\vx$; when K/V are both input-tied with Q and output-tied with each other, the input latent $\vz$ still suffices and the output tie only narrows $\mB_k$ and $\mB_v$. Either latent-cache implementation requires changes beyond replacing linear layers with two factors. No QR, modulation or direction dropout remains after export; merging uses the evaluation projector $\mP_{\alpha=1,\vm=\mathbf{1}}$ of \Cref{sec:lsp}.

\subsection{Limitations}
\label{app:limitations}
(i) The measured-KL allocation is a local surrogate: isolated unit costs omit interactions, and the current strategy has no global-optimality guarantee. This makes the allocation uncertain on quality in specific settings:
in a few cases uniform ranks can still be preferred; however empirical evidence shows that measured KL is still better, since it is able to adapt to different architectures: performance improvement over Qwen3-4B, OPT-1.3B, and two out of three compression ranges of Llama-2-7B.
(ii) Compression is also not free: \method requires both the KL measurements and projection training, although the measurements are reusable across ratios and objectives. On the other hand, compression is a one-off cost that is repaid by the inference savings of \Cref{subsec:compute} and \method justifies this with the strong performance.
(iii) Tied groups only lead to savings with latent cache if KV are tied on the input side. This means our best configuration for Qwen3-4B does not gain in efficiency as much as the input-tied version does. This represents a performance-efficiency trade-off, as shown in \Cref{tab:qwen_tying}: tying KV on the output side leads to the best performance, but one can choose to tie KV on the input side for improved inference efficiency instead.

\section{Implementation Details}
\label{app:protocol}

\myparagraph{Shared configuration.} The main accuracy tables use the measured-KL allocation of \Cref{eq:kl_cost,eq:marginal_cost} against the compressible-linear parameter count, the whitened initialization of \Cref{subsec:alloc_init} with a Cholesky whitener, shared-input tying (with the Qwen exception below), a linear-warmup cosine schedule, Adam with no weight decay, modulation ramped to $1$ over the first epoch, $\lambda_\text{ort} = 0.05$, and a factorization tolerance of $0.05$ below which singular directions are dropped at merge. Weights are loaded in bf16; compression arithmetic (covariance accumulation, Cholesky, SVD, QR) runs in fp32 on Qwen3-4B and Llama-2-7B and in fp64 on OPT-125M, OPT-1.3B and ViT-B/16.

\myparagraph{KL-measurement settings and tying.} The allocation subset $\mathcal{D}_\text{alloc}$ holds 128K calibration tokens for the LLMs and $4{,}096$ images for the ViT, and $\Delta_u$ compares the dense and the truncated model's next-token or class distributions. The removal grids are $i/8$ ($i=1,\ldots,7$) on OPT and ViT, and $i/16$ ($i=1,\ldots,15$) on Qwen3-4B and Llama-2-7B. Each measurement truncates one unit with every other unit dense, and the allocation is not refined jointly afterwards. Measuring units in isolation is an additive surrogate: on the selected allocations the joint KL of the compressed model exceeds the sum of the isolated costs by $2.1$--$4.3\times$ on Llama-2-7B, $1.7$--$3.3\times$ on Qwen3-4B, $2.2$--$4.6\times$ on the OPT models and $1.5$--$14\times$ on the ViT, generally rising with the ratio. On Qwen3-4B, K and V share an output-side factor, following the grouped-query routing, rather than the forced Q/K/V input tie; the allocation leaves them mostly dense (\Cref{tab:qwen_tying}).

\myparagraph{Replicates.} \Cref{tab:runs_seeds} provides evidence of \method robustness, reporting the standard deviations over three seeds for \method results in \Cref{sec:experiments}.
\begin{table}[t]
\centering
\caption{Seed sensitivity of \method\ and \method$^T$: mean $\pm$ standard deviation over three full-pipeline seeds, each run
read at its own validation-selected epoch. The LLMs are WikiText-2 perplexity,
\mbox{ViT-B/16} is CIFAR-100 accuracy on the single (CIFAR-100 only) pool.}
\vspace{-0.2cm}
\label{tab:runs_seeds}
\resizebox{0.7\linewidth}{!}{
\begin{tabular}{@{}lllccc@{}}
\toprule
\textbf{Model} & \textbf{Metric} & \textbf{Method} & -30\% & -50\% & -70\% \\
\midrule
\multirow{2}{*}{\textbf{OPT-125M}} & \multirow{2}{*}{ppl $\downarrow$}
 & \method$^T$ & 24.48 $\pm$ 0.11 & 30.63 $\pm$ 0.28 & 42.41 $\pm$ 0.06 \\
 & & \method & 30.95 $\pm$ 0.42 & 35.35 $\pm$ 0.67 & 48.88 $\pm$ 0.24 \\
\midrule
\multirow{2}{*}{\textbf{OPT-1.3B}} & \multirow{2}{*}{ppl $\downarrow$}
 & \method$^T$ & 12.93 $\pm$ 0.02 & 16.42 $\pm$ 0.11 & 23.33 $\pm$ 0.26 \\
 & & \method & 15.70 $\pm$ 0.15 & 18.51 $\pm$ 0.11 & 22.39 $\pm$ 0.30 \\
\midrule
\multirow{2}{*}{\textbf{Qwen3-4B}} & \multirow{2}{*}{ppl $\downarrow$}
 & \method$^T$ & 9.20 $\pm$ 0.03 & 12.48 $\pm$ 0.02 & 19.40 $\pm$ 0.14 \\
 & & \method & 9.13 $\pm$ 0.07 & 11.42 $\pm$ 0.10 & 16.18 $\pm$ 0.08 \\
\midrule
\multirow{2}{*}{\textbf{Llama-2-7B}} & \multirow{2}{*}{ppl $\downarrow$}
 & \method$^T$ & 6.39 $\pm$ 0.01 & 8.43 $\pm$ 0.06 & 13.03 $\pm$ 0.10 \\
 & & \method & 6.62 $\pm$ 0.05 & 8.05 $\pm$ 0.08 & 10.94 $\pm$ 0.12 \\
\midrule
\multirow{2}{*}{\textbf{ViT-B/16}} & \multirow{2}{*}{acc $\uparrow$}
 & \method$^T$ & 89.26 $\pm$ 0.13 & 87.80 $\pm$ 0.17 & 84.21 $\pm$ 0.06 \\
 & & \method & 89.38 $\pm$ 0.19 & 88.32 $\pm$ 0.14 & 85.92 $\pm$ 0.48 \\
\bottomrule
\end{tabular}
}
\end{table}

\myparagraph{Per-model settings.} Calibration is $1024$ sequences of $2048$ tokens for the decoder models and $10{,}000$ images for the ViT covariance pass, with a per-epoch draw of $1024$ sequences ($4096$ images) for the training stage. Batch sizes are $32$ on the OPT models, $8$ globally on Qwen3-4B and Llama-2-7B (four-way data parallel at micro-batch $1$), and $512$ on ViT. Per-direction dropout is $p = 0.05$ ($0.1$ on Llama-2-7B). Validation is a held-out split throughout, so that no test data is seen during selection.

\myparagraph{Selection.} The \emph{epoch} is selected on validation data, and the selected epoch's projections are the ones merged and reported. The \emph{learning rate} is swept, keeping the best validation setting per model and ratio, which is the same rule applied to every baseline that exposes a tunable knob. 

\myparagraph{Compression ratio.} A ratio is the fraction of the parameters of the \emph{compressible linear layers} that is removed --- the layers \method substitutes, excluding embeddings, the language-model head and the normalizations --- counted on the merged model with shared factors counted once. This is the denominator SVD-LLM, Swift-SVD and Dobi-SVD also report against, so nominal ratios are directly comparable. The measured-KL allocator greedily fills the budget and integer-trims its last move, and realized savings land within $0.25$ points of nominal in every cell ($-30.0$--$30.2$, $-50.0$--$50.2$ and $-70.0$--$70.2\%$).

\myparagraph{Baselines.} All baselines are using the same calibration budget of $1024$ sequences of $2048$ tokens, run with bf16 weights and compression arithmetic in fp32 or fp64 (following \method), and, for the SVD methods, are evaluated in factorized form. Where a baseline exposes a tuning knob we sweep it around the published value and keep the best per model and ratio: SVD-LLM's recovery learning rate against its single published one, Dobi-SVD's $\gamma$ learning rate around the released default, Swift-SVD's allocation $\alpha$ over its own eleven-point grid, and ASVD's $\alpha$. Dobi-SVD is run without its quantization step so that parameter counts match; Swift-SVD's OPT rows are bias-corrected; SVD-LLM's recovery stage is LoRA of rank $8$ on the calibration corpus of the table it appears in.
Notice that \method keeps original weights frozen, and learns to compress their subspace, which means it does not move from the original subspace but only projects.
SVD-LLM-v2 \citep{wang2025svdllmv2} and MoDeGPT \citep{lin2025modegptmodulardecompositionlarge}, though relevant, could not be reported as no official implementation has been released. The ViT ports of SVD-LLM~(W) and SliceGPT are adapted from the official released implementation, while FLAR-SVD and PELA come from their official implementation. PELA, specifically, re-trains all the network's weights with the objective to copy the dense model's intermediate features, which acts as a recovery fine-tuning stage.

\myparagraph{Evaluation.} Perplexity follows the SparseGPT and SliceGPT convention: the test split is joined, tokenized once and cut into non-overlapping $2048$-token windows. Zero-shot accuracy is measured with the EleutherAI harness \citep{evalharness} at zero shots on the merged model, and we report plain accuracy.

\myparagraph{Benchmarks.} The choice of benchmarks follows other compression methods in the literature. In \Cref{tab:results} we report perplexity for a diverse mix of pruning and low-rank compression methods, as reported by \citet{ashkboos2024slicegpt,van2023llm,wang2025svdllm,qinsi2025dobisvd,qi2026swiftsvd}. For zero-shot performance (\Cref{tab:zeroshot_6bench}) we restrict the analysis to low-rank compression methods \citep{wang2025svdllm,qinsi2025dobisvd,qi2026swiftsvd}, since they are the closest to \method. Specifically, we report the same six benchmark evaluations as baselines do, and calibrate on the same datasets. While Alpaca-calibrated compression follows our compression ratios reported for other tables, for C4-calibrated compression we follow same procedure as Swift-SVD, with compression ratios of $-20\%/40\%/60\%$, to introduce further richness in compression ratios. We re-evaluate all methods and benchmarks for evaluation fairness, using the same setup for calibration and evaluation for all methods.

\section{Initialization, Tying and Allocation}
\label{app:ablations}

We ablate three choices: the spectral initialization, sharing a projector across a tied group, and allocating ranks from output KL. The initialization and tying ablations use a shape-based uniform allocation, so that within each table only the stated factor changes.

\myparagraph{Uniform allocation.} The uniform control assigns every compressed matrix the same retained fraction $\rho$ of its dense parameter count, before accounting for ties:
\begin{equation}
    r_\ell = \Big\lfloor \rho\,\frac{d_\text{in}d_\text{out}}{d_\text{in}+d_\text{out}}\Big\rfloor,
    \qquad k_\ell = d_\ell-r_\ell,
    \label{eq:unif_alloc}
\end{equation}
where $d_\ell$ is the projected dimension: $\min(d_\text{in},d_\text{out})$ for an individual layer, $d_\text{in}$ for an input-tied group and $d_\text{out}$ for an output-tied group. A tied group uses $k_\mathcal{G}=\min_{\ell\in\mathcal{G}}k_\ell$, and $\rho$ is set by binary search on the merged parameter count, shared factors counted once.

\myparagraph{Initialization.} Before training, neither start dominates (\Cref{tab:init_untrained}): the whitened truncation is better on OPT-1.3B and Llama-2-7B, the plain Gram on Qwen3-4B and on OPT-125M at $-30$ and $-50\%$, and all cells are far from dense. The untrained ordering does not carry over to training: on OPT-125M, where the plain start leads by up to $2\times$ untrained, the whitened start ends about $2\%$ lower in perplexity at every ratio (\Cref{tab:init_trained_opt125m}), and on Llama-2-7B it also stays ahead after training (\Cref{tab:init_tying_llama}). We therefore keep the whitened truncation.

\begin{table}[b!]
\captionsetup{font=small}
\centering
\caption{\textbf{Initialization before training.} WikiText-2 test perplexity of the training-free truncation at uniform allocation, from the whitened truncation of \Cref{subsec:alloc_init} (\emph{Whitened}, NoLSP at uniform allocation) or from the trailing directions of the plain activation Gram (\emph{Plain}), removing the same number of parameters per unit.}
\label{tab:init_untrained}
\small
\setlength{\tabcolsep}{5pt}
\begin{tabular}{@{}l c cc cc cc@{}}
\toprule
 & & \multicolumn{2}{c}{$\mathbf{-30\%}$} & \multicolumn{2}{c}{$\mathbf{-50\%}$} & \multicolumn{2}{c}{$\mathbf{-70\%}$} \\
\cmidrule(lr){3-4}\cmidrule(lr){5-6}\cmidrule(lr){7-8}
\textbf{Model} & \textbf{Dense} & \textbf{Whitened} & \textbf{Plain} & \textbf{Whitened} & \textbf{Plain} & \textbf{Whitened} & \textbf{Plain} \\
\midrule
OPT-125M & $27.6$ & $582$ & $\mathbf{282}$ & $664$ & $\mathbf{534}$ & $\mathbf{1{,}538}$ & $2{,}211$ \\
OPT-1.3B & $14.6$ & $\mathbf{26.1}$ & $52.5$ & $\mathbf{74.4}$ & $571$ & $\mathbf{814}$ & $18{,}786$ \\
Qwen3-4B & $7.9$ & $40.1$ & $\mathbf{29.4}$ & $255$ & $\mathbf{125}$ & $7{,}186$ & $\mathbf{4{,}449}$ \\
Llama-2-7B & $5.5$ & $\mathbf{9.2}$ & $13.8$ & $\mathbf{22.1}$ & $36.9$ & $\mathbf{172}$ & $215$ \\
\bottomrule
\end{tabular}

\end{table}

\begin{table}[b!]
\captionsetup{font=small}
\centering
\caption{\textbf{Initialization before and after training, OPT-125M.} WikiText-2 test perplexity (dense $27.6$) of the training-free truncation and of \method distilled from it, only the initialization changed.}
\label{tab:init_trained_opt125m}
\small
\setlength{\tabcolsep}{5pt}
\begin{tabular}{@{}l cc cc cc@{}}
\toprule
 & \multicolumn{2}{c}{$\mathbf{-30\%}$} & \multicolumn{2}{c}{$\mathbf{-50\%}$} & \multicolumn{2}{c}{$\mathbf{-70\%}$} \\
\cmidrule(lr){2-3}\cmidrule(lr){4-5}\cmidrule(lr){6-7}
\textbf{OPT-125M} & \textbf{Whitened} & \textbf{Plain} & \textbf{Whitened} & \textbf{Plain} & \textbf{Whitened} & \textbf{Plain} \\
\midrule
Untrained & $582$ & $\mathbf{282}$ & $664$ & $\mathbf{534}$ & $\mathbf{1{,}538}$ & $2{,}211$ \\
Trained (\method) & $\mathbf{31.7}$ & $32.3$ & $\mathbf{35.9}$ & $36.6$ & $\mathbf{44.8}$ & $45.8$ \\
\bottomrule
\end{tabular}

\end{table}

\myparagraph{Tying.} \Cref{tab:init_tying_llama} crosses initialization and tying on Llama-2-7B, with the tied and separate arms searched to the same parameter target (realized savings within $0.15$ points). The whitened start leads under both tyings, by $1.9$--$3.0$ Avg6 points and $20$--$33\%$ in perplexity. Tying raises Avg6 at every ratio for both starts ($0.3$--$3.1$ points, most at $-30\%$), while perplexity is mixed beyond $-30\%$. A tied group stores its input factor once, so it keeps a higher rank at a fixed budget, and its shared latent is what the cache advantage of \Cref{subsec:compute} rests on.

\begin{table}[t!]
\captionsetup{font=small}
\centering
\caption{\textbf{Initialization and tying after training, Llama-2-7B.} \method distilled on Alpaca, uniform removed rank per layer, the better of two learning rates per arm; WikiText-2 test perplexity and mean zero-shot accuracy (Avg6) of the validation-selected epoch.}
\label{tab:init_tying_llama}
\small
\setlength{\tabcolsep}{5pt}
\begin{tabular}{@{}l l cc cc cc@{}}
\toprule
 & & \multicolumn{2}{c}{$\mathbf{-30\%}$} & \multicolumn{2}{c}{$\mathbf{-50\%}$} & \multicolumn{2}{c}{$\mathbf{-70\%}$} \\
\cmidrule(lr){3-4}\cmidrule(lr){5-6}\cmidrule(lr){7-8}
\textbf{Initialization} & \textbf{Q/K/V, gate/up} & \textbf{PPL} & \textbf{Avg6} & \textbf{PPL} & \textbf{Avg6} & \textbf{PPL} & \textbf{Avg6} \\
\midrule
Whitened & tied & $\mathbf{12.2}$ & $\mathbf{52.5}$ & $21.8$ & $\mathbf{47.3}$ & $\mathbf{37.9}$ & $\mathbf{42.0}$ \\
Whitened & untied & $14.8$ & $50.4$ & $\mathbf{21.3}$ & $47.0$ & $38.2$ & $41.5$ \\
Plain & tied & $16.4$ & $50.6$ & $30.7$ & $45.3$ & $49.8$ & $39.3$ \\
Plain & untied & $22.2$ & $47.5$ & $27.9$ & $44.0$ & $47.6$ & $38.8$ \\
\bottomrule
\end{tabular}

\end{table}

\myparagraph{Allocation.} \Cref{tab:alloc_ablation} compares the measured-KL allocation on OPT-1.3B with three alternatives searched to the same realized savings: the uniform retained parameters of \Cref{eq:unif_alloc}, a uniform removed rank (the same fraction of directions removed from every unit), and a budget read off a diagonal Fisher of the calibration loss. Measured KL is best on both metrics at every ratio, by $0.2$--$0.9$ perplexity and $0.7$--$1.0$ Avg6 points over the best alternative.

\begin{table}[t!]
\captionsetup{font=small}
\centering
\caption{\textbf{Rank allocation, OPT-1.3B.} Four budgets at the same realized savings (within $0.05$ points), whitened tied initialization, distilled on WikiText-2, the better of two learning rates per arm; WikiText-2 test perplexity and Avg6 of the validation-selected epoch.}
\label{tab:alloc_ablation}
\small
\setlength{\tabcolsep}{5pt}
\begin{tabular}{@{}l cc cc cc@{}}
\toprule
 & \multicolumn{2}{c}{$\mathbf{-30\%}$} & \multicolumn{2}{c}{$\mathbf{-50\%}$} & \multicolumn{2}{c}{$\mathbf{-70\%}$} \\
\cmidrule(lr){2-3}\cmidrule(lr){4-5}\cmidrule(lr){6-7}
\textbf{Allocation} & \textbf{PPL} & \textbf{Avg6} & \textbf{PPL} & \textbf{Avg6} & \textbf{PPL} & \textbf{Avg6} \\
\midrule
Uniform retained parameters (\Cref{eq:unif_alloc}) & $16.5$ & $41.1$ & $18.7$ & $37.8$ & $24.2$ & $35.0$ \\
Uniform removed rank & $16.8$ & $40.6$ & $18.9$ & $37.6$ & $23.4$ & $35.4$ \\
Diagonal Fisher & $16.7$ & $41.4$ & $18.6$ & $38.1$ & $25.6$ & $34.5$ \\
Measured KL (\textbf{ours}) & $\mathbf{15.6}$ & $\mathbf{42.4}$ & $\mathbf{17.9}$ & $\mathbf{38.8}$ & $\mathbf{23.2}$ & $\mathbf{36.3}$ \\
\bottomrule
\end{tabular}

\end{table}

\myparagraph{Measured-KL against uniform allocation.} \Cref{tab:alloc_unifkl} compares measured-KL and uniform allocation on zero-shot accuracy, for the models and calibration sets of \Cref{tab:zeroshot_6bench}. Measured allocation leads in mean accuracy in five of six settings, by up to $3.8$ points (Qwen3-4B, $-20\%$), and trails by $0.4$ points on Llama-2-7B at $-50\%$. Per benchmark it leads in $26$ of $36$ cells, including ARC-e and PIQA in every setting, while WinoGrande favours the uniform rule in four of six. On Qwen3-4B the measured runs also change the K/V routing, which \Cref{tab:qwen_tying} separates from the allocation. The gain comes at the cost of the KL measurements (\Cref{app:cost}), although one set of measurements serves every ratio and both objectives; the effect on the deployed shape is given in \Cref{tab:inference_alloc_llama}.

\begin{table}[t!]
\captionsetup{font=small}
\centering
\caption{\textbf{Uniform against measured-KL allocation, zero-shot accuracy} ($\uparrow$, \%). \emph{Uniform} is the retained fraction of \Cref{eq:unif_alloc}; \emph{Measured KL} is the per-unit KL cost of \Cref{eq:kl_cost}. Same whitened initialization, distillation objective (\method) and realized budget; the Qwen3-4B pair also differs in K/V routing (\Cref{tab:qwen_tying}). Llama-2-7B is calibrated on Alpaca and compressed by $-30/50/70\%$, Qwen3-4B on C4 by $-20/40/60\%$, as in \Cref{tab:zeroshot_6bench}.}
\label{tab:alloc_unifkl}
\small
\setlength{\tabcolsep}{5pt}
\begin{tabular}{@{}lcccccccc@{}}
\toprule
\textbf{Allocation} & \textbf{Comp.} & \textbf{ARC-e} & \textbf{PIQA} & \textbf{Openb.} & \textbf{WinoG.} & \textbf{HellaS.} & \textbf{MathQA} & \textbf{Avg.} \\
\midrule
\multicolumn{9}{@{}l}{\textbf{Llama-2-7B} (Alpaca calibration)} \\
\midrule
Dense (Base) & -- & 75.6 & 77.8 & 32.8 & 69.9 & 57.1 & 28.1 & 56.9 \\
\midrule
Uniform & \multirow{2}{*}{-30\%} & 70.3 & 74.5 & \textbf{30.6} & 63.4 & 50.0 & 25.5 & 52.4 \\
Measured KL &  & \textbf{71.2} & \textbf{75.4} & 30.0 & \textbf{64.0} & \textbf{50.4} & \textbf{26.4} & \textbf{52.9} \\
\midrule
Uniform & \multirow{2}{*}{-50\%} & 64.3 & 70.6 & \textbf{25.8} & \textbf{59.8} & \textbf{43.8} & \textbf{24.8} & \textbf{48.2} \\
Measured KL &  & \textbf{65.2} & \textbf{70.8} & 24.6 & 58.6 & 43.4 & 23.9 & 47.8 \\
\midrule
Uniform & \multirow{2}{*}{-70\%} & 53.2 & 64.4 & 19.0 & \textbf{53.4} & 35.3 & \textbf{22.1} & 41.2 \\
Measured KL &  & \textbf{55.6} & \textbf{66.1} & \textbf{21.8} & 51.7 & \textbf{36.1} & 21.9 & \textbf{42.2} \\
\midrule
\multicolumn{9}{@{}l}{\textbf{Qwen3-4B} (C4 calibration)} \\
\midrule
Dense (Base) & -- & 79.0 & 77.9 & 32.0 & 70.3 & 54.6 & 53.8 & 61.3 \\
\midrule
Uniform & \multirow{2}{*}{-20\%} & 74.1 & 74.4 & 26.0 & 64.2 & 49.6 & 32.1 & 53.4 \\
Measured KL &  & \textbf{78.3} & \textbf{76.7} & \textbf{30.8} & \textbf{65.6} & \textbf{51.8} & \textbf{40.0} & \textbf{57.2} \\
\midrule
Uniform & \multirow{2}{*}{-40\%} & 67.2 & 71.4 & 25.2 & \textbf{61.8} & 44.7 & 26.0 & 49.4 \\
Measured KL &  & \textbf{70.6} & \textbf{72.3} & \textbf{28.8} & 60.9 & \textbf{45.4} & \textbf{26.5} & \textbf{50.7} \\
\midrule
Uniform & \multirow{2}{*}{-60\%} & 55.2 & 67.8 & 20.6 & \textbf{55.5} & \textbf{38.3} & 22.6 & 43.3 \\
Measured KL &  & \textbf{56.9} & \textbf{68.8} & \textbf{21.2} & 55.3 & 37.3 & \textbf{23.3} & \textbf{43.8} \\
\bottomrule
\end{tabular}

\end{table}

\myparagraph{K/V routing on Qwen3-4B.} \Cref{tab:qwen_tying} crosses the two allocations with forced Q/K/V input tying and with the grouped-query routing (K and V share an output-side factor, Q is projected alone), at fixed objective, learning rate, epoch budget and realized savings. Under the uniform allocation, routing changes Avg6 by at most $0.6$ points and narrows the cache by about $8\%$. Under the measured allocation, the output-side routing is ahead at every ratio ($0.5$--$1.6$ Avg6 points, $2$--$5\%$ in perplexity), but K and V stay almost uncompressed, so the cache is as wide as the dense model's at $-20$ and $-40\%$. At matched input tying, measured allocation leads by $3.0$ Avg6 points at $-20\%$ and trails by $0.1$ and $0.4$ at $-40$ and $-60\%$. The wide measured cache follows from the KL measurements: per parameter saved, truncating the K/V pair costs about an order of magnitude more KL than any other unit type, so greedy allocation buys it last, whereas the uniform rule compresses every unit.

\begin{table}[H]
\captionsetup{font=small}
\centering
\caption{\textbf{Allocation and K/V routing, Qwen3-4B.} \method distilled on C4, same learning rate, epoch budget and realized savings; gate/up are tied in every arm. \emph{Q/K/V, input}: one shared input projection for Q, K and V. \emph{K/V, output}: the grouped-query routing, with K and V sharing an output-side factor and Q projected alone. The uniform Q/K/V row is that table's Uniform cell. Avg6 and C4 test perplexity of the validation-selected epoch; KV is the floats cached per token per layer (dense $2048$). Best per column in bold.}
\label{tab:qwen_tying}
\small
\setlength{\tabcolsep}{4pt}
\label{tab:qwen4_allocation_performance}
\begin{tabular}{@{}ll ccc ccc ccc@{}}
\toprule
 & & \multicolumn{3}{c}{$\mathbf{-20\%}$} & \multicolumn{3}{c}{$\mathbf{-40\%}$} & \multicolumn{3}{c}{$\mathbf{-60\%}$} \\
\cmidrule(lr){3-5}\cmidrule(lr){6-8}\cmidrule(lr){9-11}
\textbf{Allocation} & \textbf{Tied group} & \textbf{Avg6} & \textbf{PPL} & \textbf{KV} & \textbf{Avg6} & \textbf{PPL} & \textbf{KV} & \textbf{Avg6} & \textbf{PPL} & \textbf{KV} \\
\midrule
Uniform & Q/K/V, input & $53.4$ & $15.9$ & $1357$ & $49.4$ & $18.4$ & $1018$ & $43.3$ & $23.0$ & $679$ \\
Uniform & K/V, output & $54.0$ & $16.0$ & $\mathbf{1242}$ & $49.3$ & $18.6$ & $\mathbf{932}$ & $42.9$ & $23.6$ & $\mathbf{622}$ \\
Measured KL & Q/K/V, input & $56.4$ & $15.0$ & $1822$ & $49.3$ & $17.9$ & $1701$ & $42.9$ & $23.2$ & $1282$ \\
Measured KL & K/V, output & $\mathbf{56.9}$ & $\mathbf{14.7}$ & $2048$ & $\mathbf{50.9}$ & $\mathbf{17.1}$ & $2048$ & $\mathbf{44.2}$ & $\mathbf{22.4}$ & $1920$ \\
\bottomrule
\end{tabular}

\end{table}

\newpage
\section{What \method Compresses, Across Models}
\label{app:subspace}

\Cref{fig:rank_profile_models} repeats the allocation measurement of \Cref{subsec:analysis} on the measured-KL checkpoints of OPT-125M, OPT-1.3B, Qwen3-4B and the CIFAR-100 ViT-B/16, and \Cref{fig:drop_diff_blocks} repeats its spectral measurement on single blocks of OPT-1.3B and Llama-2-7B.

\myparagraph{Where the budget goes.} On OPT-125M, Q/K/V blocks are the least compressed at all compression budgets, especially the last layers, which stay uncompressed at all compression ratios. For both OPT models, the last layers tend to be less compressed than earlier layers. For Qwen3-4B, an interesting behavior appears: because of non-tying Q and K/V together, Q becomes the most compressed layer at all compression ratios, while K/V layers are almost never compressed (only for few early layer, at $-70\%$ compression). In addition, MLP layers are also compressed less, especially the down-projection one. The ViT instead shows a pattern where middle layers get compressed less, at all compression ratios.

\begin{figure}[H]
\centering
\includegraphics[width=0.49\linewidth]{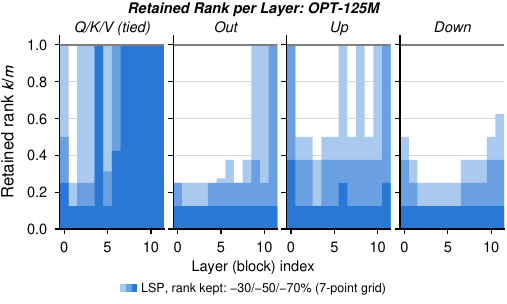}\hfill
\includegraphics[width=0.49\linewidth]{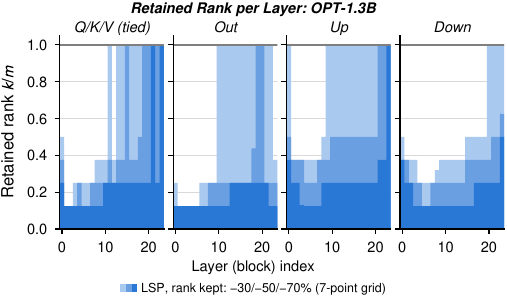}\\[4pt]
\includegraphics[width=0.545\linewidth]{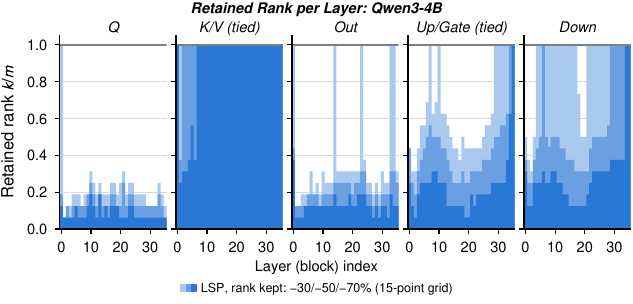}\hfill
\includegraphics[width=0.435\linewidth]{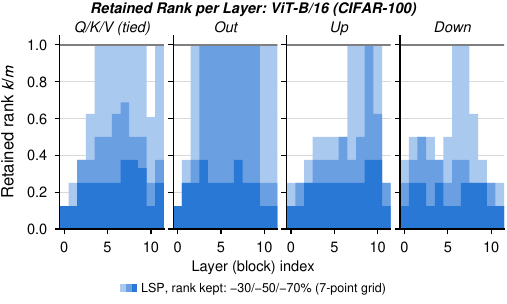}
\vspace{-0.2cm}
\caption{\textbf{Rank allocation across models.} Retained rank $r$ as a fraction of the full rank $d$, by block and projection type, for the measured-KL \method checkpoints of OPT-125M, OPT-1.3B, Qwen3-4B and the CIFAR-100 ViT-B/16; the same panel for Llama-2-7B is \Cref{fig:drop_diff} (left). Shades stack the three ratios, $-30\%$ lightest behind and $-70\%$ darkest in front, so each shade's top edge is that ratio's profile wherever the allocations nest; $r/d = 1$ is a unit the allocator left dense. The removal grid is $7$ points on OPT-125M, OPT-1.3B and the ViT and $15$ on Qwen3-4B. Qwen3-4B has grouped-query attention, so Q is projected alone and K/V share an output-side factor, and each gets its own panel.}
\label{fig:rank_profile_models}
\vspace{-0.2cm}
\end{figure}
\myparagraph{What the removal looks like, block by block.} \Cref{fig:drop_diff_blocks} repeats the two spectral panels of \Cref{fig:drop_diff} for an early, a middle and a late block of OPT-1.3B (blocks $3$, $12$ and $21$ of $24$) and Llama-2-7B (blocks $4$, $16$ and $28$ of $32$).
Removal looks as in Llama's average: it spreads over the whole spectrum rather than the tail a weight SVD would cut, deepens with the ratio, and reaches the first singular direction, most strongly in the up projection on OPT-1.3B (direction~$1$ loses $2.5$--$4.4$ of its amplitude in block~$3$) and in the tied Q/K/V on Llama-2-7B ($0.7$--$2.3$).
Deeper blocks are spared more: OPT-1.3B's block~$21$ keeps Q/K/V dense at every ratio and its other projections almost intact at $-30\%$, and on Llama-2-7B block~$16$ keeps its Q/K/V, output and down projections dense at $-30\%$ and block~$28$ its up and down projections.
The difference to NoLSP is small in every block, so the thin difference of the averaged measurement is not an averaging artifact: at most $2.5\%$ of the removal in the Llama-2-7B blocks shown (median over all compressed units $1.0\%$), and on OPT-1.3B growing with depth, from $1$--$2\%$ in block~$3$ to $5.8\%$ in the down projection of block~$21$ (median $1.6\%$).
Its sign pattern holds across depth: \method removes more of the leading quarter of directions and keeps more of the trailing half than NoLSP in $171$ of the $180$ compressed OPT-1.3B units and $244$ of the $248$ Llama-2-7B units at $-50\%$ and $-70\%$.

\begin{figure}[p]
\centering
{\small\textbf{OPT-1.3B}}\\[4pt]
\includegraphics[width=0.33\linewidth]{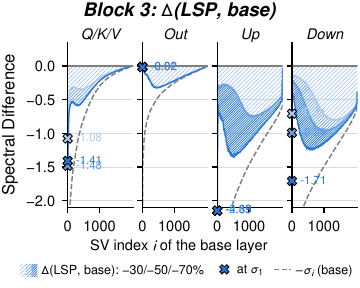}\hfill
\includegraphics[width=0.33\linewidth]{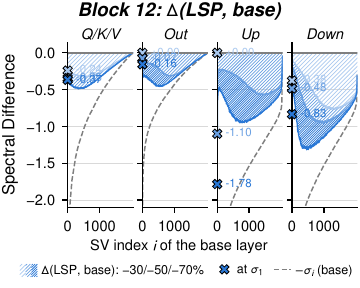}\hfill
\includegraphics[width=0.33\linewidth]{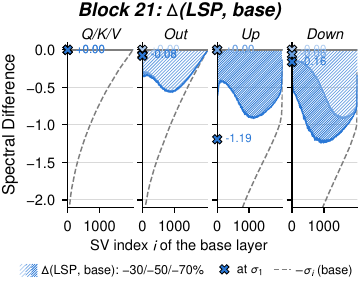}\\[2pt]
\includegraphics[width=0.33\linewidth]{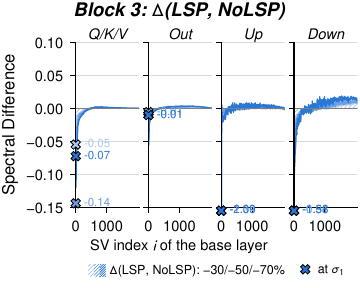}\hfill
\includegraphics[width=0.33\linewidth]{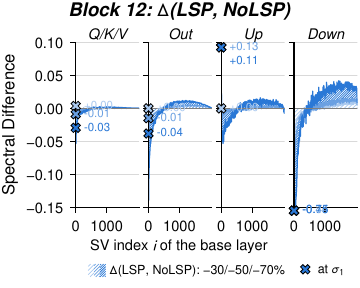}\hfill
\includegraphics[width=0.33\linewidth]{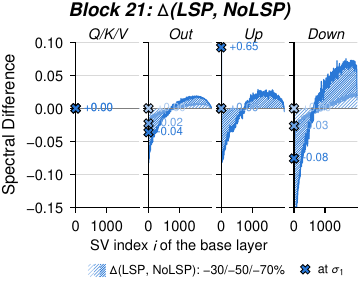}\\[6pt]
{\small\textbf{Llama-2-7B}}\\[4pt]
\includegraphics[width=0.33\linewidth]{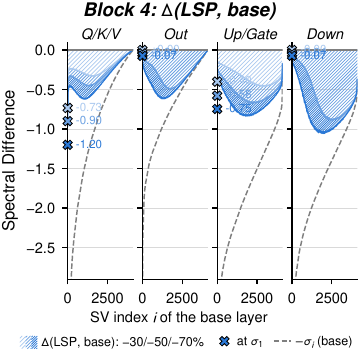}\hfill
\includegraphics[width=0.33\linewidth]{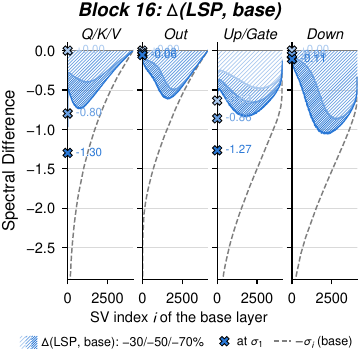}\hfill
\includegraphics[width=0.33\linewidth]{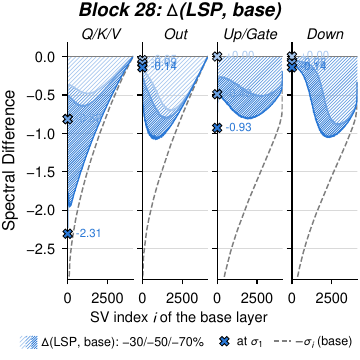}\\[2pt]
\includegraphics[width=0.33\linewidth]{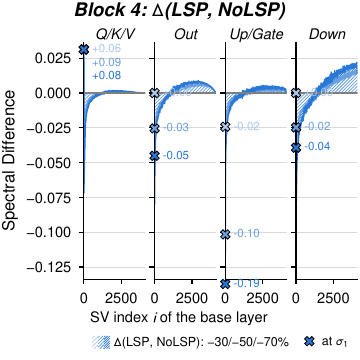}\hfill
\includegraphics[width=0.33\linewidth]{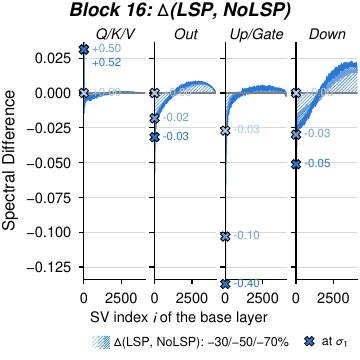}\hfill
\includegraphics[width=0.33\linewidth]{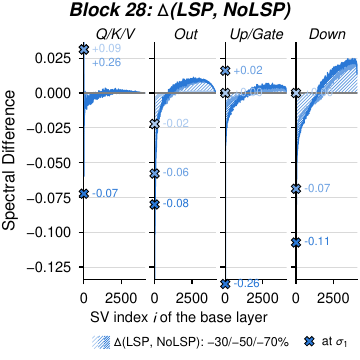}
\vspace{-0.2cm}
\caption{\textbf{Spectral removal in single blocks.} The two spectral panels of \Cref{fig:drop_diff} for an early, a middle and a late block (columns) of OPT-1.3B (blocks $3$, $12$, $21$ of $24$, top) and Llama-2-7B (blocks $4$, $16$, $28$ of $32$, bottom), on the measured-KL checkpoints of \Cref{tab:results}. For each model, the first row is the amplitude lost along each original singular direction, $\Delta_i=\sigma_i(a_i-1)$, with $-\sigma_i$ of the uncompressed layer dashed, and the second the difference to NoLSP, $\sigma_i(a_i^{\text{LSP}}-a_i^{\text{NoLSP}})$, at the same measured-KL ranks. Each row shares one vertical axis. Shades are the $-30/50/70\%$ runs (light to dark), the bands are a running mean over $9$ directions, a missing band is a unit left dense at that ratio, and tied Q/K/V and up/gate panels average their member matrices; crosses give the raw value on direction~1, parked on the axis edge when outside it.}
\label{fig:drop_diff_blocks}
\vspace{-0.2cm}
\end{figure}

\newpage
\section{Inference Efficiency}
\label{app:inference_models}

\Cref{tab:inference_ratios} extends \Cref{tab:inference} to all three ratios on the WikiText-2-calibrated Llama-2-7B checkpoints. Throughout this section, throughput is measured on one GH200 in bf16 with SDPA and CUDA-graph decoding, at batch 1 and with a full-width KV cache, while latent-cache widths and context capacities are analytical estimates computed from the factor widths, assuming a latent cache (\Cref{app:merge}) and counting persistent tensors only; context capacities assume a 95.5\,GiB device with 6\,GiB headroom. At weights and FLOPs matched to within $2\%$, \method decodes fastest at every ratio, $15$--$26\%$ ahead of the best untied factorization at $512$ tokens, and its latent cache is $1.2\times$, $1.9\times$ and $2.2\times$ narrower. Every untied baseline factorization decodes slower than the dense model at $-30\%$, while \method is faster than dense at every ratio ($1.21\times$, $1.36\times$ and $1.56\times$ at $-30\%$, $-50\%$ and $-70\%$). This margin is a short-context one: it falls to $1.11\times$ at 16k tokens, where reading the cache dominates.

\begin{table}[b!]
\centering
\caption{Inference efficiency of compressed Llama-2-7B at every ratio, the checkpoints of \Cref{tab:results}, on one GH200 (bf16, SDPA). \emph{Memory}, analytic from the checkpoints' ranks (factors vs.\ dense): weights, floats cached per token per layer with the latent KV cache ($\vz = \mA\vx$; a unit the measured-KL allocation leaves uncompressed stays dense and caches full-width K/V), weights plus KV cache at 128k tokens and batch 8, and the longest context that fits 95.5\,GiB at batch 8. \emph{Speed}: FLOPs to prefill 2k tokens and per decode step at 2k context, and measured CUDA-graph-compiled decode tok/s at batch 1 for 512, 4k and 16k tokens of context. Best factorized method per ratio in bold.}
\label{tab:inference_ratios}
\small
\resizebox{\linewidth}{!}{
\setlength{\tabcolsep}{1.5pt}
\begin{tabular}{@{}lcccccccccc@{}}
\toprule
& \multicolumn{4}{c}{\textbf{Memory}} & \multicolumn{5}{c}{\textbf{Compute and speed (batch 1)}} \\
\cmidrule(lr){2-5}\cmidrule(lr){6-10}
\textbf{Method} & \textbf{Weights} & \textbf{KV floats} & \textbf{Weights+KV} & \textbf{Max ctx} & \textbf{Prefill} & \textbf{Decode} & \multicolumn{3}{c}{\textbf{Decode tok/s}} \\
& \textbf{(GiB)} & \textbf{/token/layer} & \textbf{@128k, b8 (GiB)} & \textbf{(k tok, b8)} & \textbf{TFLOPs @2k} & \textbf{GFLOP/step} & \textbf{@512} & \textbf{@4k} & \textbf{@16k} \\
\midrule
Llama-2-7B (dense) & 12.55 & 8192 & 524.6 & 19.7 & 29.26 & 14.29 & 142 & 85 & 36 \\
\midrule
\multicolumn{10}{@{}l}{\textit{$-30\%$}} \\
SVD-LLM (W) & 8.93 & 2866 ($2.9\times$) & 188.1 ($2.8\times$) & 59.0 & 21.30 & 10.40 & 135 & 83 & 35 \\
Swift-SVD & 8.93 & 2866 ($2.9\times$) & 188.1 ($2.8\times$) & 59.0 & 21.30 & 10.40 & 137 & 83 & 36 \\
Dobi-SVD & 9.06 & 3279 ($2.5\times$) & 214.0 ($2.5\times$) & 51.4 & 21.59 & 10.54 & 136 & 83 & 35 \\
\textbf{LSP (Ours)} & 8.93 & \textbf{2316 ($3.5\times$)} & \textbf{153.7 ($3.4\times$)} & \textbf{73.0} & 21.30 & 10.40 & \textbf{172} & \textbf{95} & \textbf{37} \\
\midrule
\multicolumn{10}{@{}l}{\textit{$-50\%$}} \\
SVD-LLM (W) & 6.52 & 2048 ($4.0\times$) & 134.5 ($3.9\times$) & 85.0 & 16.00 & 7.81 & 163 & 92 & 37 \\
Swift-SVD & 6.52 & 2048 ($4.0\times$) & 134.5 ($3.9\times$) & 85.0 & 16.00 & 7.81 & 164 & 92 & 37 \\
Dobi-SVD & 6.53 & 2060 ($4.0\times$) & 135.2 ($3.9\times$) & 84.5 & 16.01 & 7.82 & 164 & 93 & 37 \\
\textbf{LSP (Ours)} & 6.52 & \textbf{1096 ($7.5\times$)} & \textbf{75.0 ($7.0\times$)} & \textbf{158.8} & 16.00 & 7.81 & \textbf{193} & \textbf{101} & \textbf{38} \\
\midrule
\multicolumn{10}{@{}l}{\textit{$-70\%$}} \\
SVD-LLM (W) & 4.11 & 1228 ($6.7\times$) & 80.9 ($6.5\times$) & 145.8 & 10.69 & 5.22 & 191 & 100 & 38 \\
Swift-SVD & 4.11 & 1228 ($6.7\times$) & 80.9 ($6.5\times$) & 145.8 & 10.69 & 5.22 & 193 & 101 & 38 \\
Dobi-SVD & 4.12 & 1434 ($5.7\times$) & 93.7 ($5.6\times$) & 124.9 & 10.71 & 5.23 & 189 & 99 & 38 \\
\textbf{LSP (Ours)} & 4.11 & \textbf{556 ($14.7\times$)} & \textbf{38.9 ($13.5\times$)} & \textbf{322.1} & 10.69 & 5.22 & \textbf{221} & \textbf{108} & \textbf{40} \\
\bottomrule
\end{tabular}
}
\end{table}

\myparagraph{Allocation at deployment.} \Cref{tab:inference_alloc_llama} compares the uniform and measured-KL allocations on the same checkpoints; at a given ratio both hold the same weights ($8.93$, $6.52$ and $4.11$\,GiB) and decode FLOPs. The measured budget decodes faster at every ratio ($10$, $17$ and $6\%$), partly because a unit it leaves dense runs as one matrix multiply instead of two. Its cache is $38\%$ wider at $-30\%$, where a few Q/K/V groups stay uncompressed and cache full-width K and V, but $8$ and $22\%$ narrower at $-50$ and $-70\%$, where it compresses attention harder. Both allocations hold longer contexts and decode faster than every untied factorization of \Cref{tab:inference_ratios}, so the advantage over them does not depend on the allocation.

\begin{table}[t!]
\captionsetup{font=small}
\centering
\caption{\textbf{Deployment cost of the allocation, Llama-2-7B.} The uniform and measured-KL allocations compared in \Cref{tab:alloc_unifkl}, here on the WikiText-2-calibrated checkpoints, on one GH200, Q/K/V tied on the input side. \emph{KV}: floats cached per token per layer, and \emph{Max ctx}: the longest context (k tokens) that fits $95.5$\,GiB at batch~8, both analytic from the checkpoints' ranks, counting a unit the allocator left uncompressed as caching full-width K,V. \emph{tok/s}: measured CUDA-graph-compiled decode throughput at batch 1 and $512$ tokens of context, both allocations in the same session, best of two passes. Within a ratio both hold the same weights and issue the same decode FLOPs.}
\label{tab:inference_alloc_llama}
\small
\setlength{\tabcolsep}{4pt}
\begin{tabular}{@{}l ccc ccc ccc@{}}
\toprule
 & \multicolumn{3}{c}{$\mathbf{-30\%}$} & \multicolumn{3}{c}{$\mathbf{-50\%}$} & \multicolumn{3}{c}{$\mathbf{-70\%}$} \\
\cmidrule(lr){2-4}\cmidrule(lr){5-7}\cmidrule(lr){8-10}
\textbf{Allocation} & \textbf{KV} & \textbf{Max ctx} & \textbf{tok/s} & \textbf{KV} & \textbf{Max ctx} & \textbf{tok/s} & \textbf{KV} & \textbf{Max ctx} & \textbf{tok/s} \\
\midrule
Llama-2-7B (dense) & $8192$ & $19.7$ & $142$ & $8192$ & $19.7$ & $142$ & $8192$ & $19.7$ & $141$ \\
\midrule
Uniform & $\mathbf{1673}$ & $\mathbf{101.0}$ & $157$ & $1195$ & $145.6$ & $165$ & $717$ & $249.8$ & $209$ \\
Measured KL & $2316$ & $73.0$ & $\mathbf{172}$ & $\mathbf{1096}$ & $\mathbf{158.8}$ & $\mathbf{193}$ & $\mathbf{556}$ & $\mathbf{322.1}$ & $\mathbf{221}$ \\
\bottomrule
\end{tabular}

\end{table}

\myparagraph{K/V routing at deployment.} \Cref{tab:qwen_tying_deploy} gives the deployment cost of the four Qwen3-4B arms of \Cref{tab:qwen_tying}. Within a ratio they hold the same weights ($6.14$, $4.78$ and $3.43$\,GiB) and issue the same decode FLOPs ($7.80$, $6.35$ and $4.89$\,GFLOP per step). The allocation sets the speed: the measured budget is faster than the uniform one in all six pairs ($22$, $10$ and $8\%$ at $-20\%$, $-40\%$ and $-60\%$) because it leaves units dense, whereas the uniform rule factorizes every layer and decodes slower than the dense model at $-20$ and $-40\%$. Input tying is $2$--$4\%$ faster than the output-side routing in all six pairs, since one shared input factor serves Q, K and V in a single matrix multiply. The routing mainly decides the cache, in a direction set by the allocation: it narrows the uniform cache by about $8\%$ and widens the measured one to or near the dense width. No configuration wins on all three axes: uniform with output-side K/V holds the longest context but is the slowest, measured with input tying is the fastest, and measured with output-side K/V is the most accurate but caches nearly as much as the dense model.

Measured-KL allocation with K/V tied on the output side is the most accurate configuration at every ratio (\cref{tab:qwen_tying}), so we use it as the default. When deployment efficiency matters more, \method supports the other trade-offs: uniform allocation with output-side K/V holds the longest context, and measured allocation with input-side tying decodes fastest.

\begin{table}[H]
\captionsetup{font=small}
\centering
\caption{\textbf{Deployment cost of allocation and K/V routing, Qwen3-4B.} The four arms of \Cref{tab:qwen_tying} on one GH200. \emph{KV}: floats cached per token per layer, and \emph{Max ctx}: the longest context that fits $95.5$\,GiB at batch~8, both analytic from the merged checkpoints' ranks, counting a unit the allocator left uncompressed as caching full-width K,V. \emph{tok/s}: measured CUDA-graph-compiled decode throughput at batch 1 and $512$ tokens of context, best of two passes. Within a ratio all four arms hold the same weights and issue the same decode FLOPs, given in the text.}
\label{tab:qwen_tying_deploy}
\small
\setlength{\tabcolsep}{4pt}
\vspace{-0.2cm}
\setlength{\tabcolsep}{2pt}
\resizebox{\linewidth}{!}{
\begin{tabular}{@{}ll cccc cccc cccc@{}}
\toprule
 & & \multicolumn{4}{c}{$\mathbf{-20\%}$} & \multicolumn{4}{c}{$\mathbf{-40\%}$} & \multicolumn{4}{c}{$\mathbf{-60\%}$} \\
\cmidrule(lr){3-6}\cmidrule(lr){7-10}\cmidrule(lr){11-14}
\textbf{Allocation} & \textbf{Tied group} & \textbf{KV} & \textbf{Max ctx} & \textbf{tok/s} & \textbf{Avg6} & \textbf{KV} & \textbf{Max ctx} & \textbf{tok/s} & \textbf{Avg6} & \textbf{KV} & \textbf{Max ctx} & \textbf{tok/s} & \textbf{Avg6} \\
\midrule
\multicolumn{2}{@{}l}{Qwen3-4B (dense)} & $2048$ & $74.6$ & $175$ & $61.3$ & $2048$ & $74.6$ & $175$ & $61.3$ & $2048$ & $74.6$ & $175$ & $61.3$ \\
\midrule
Uniform & Q/K/V, input & $1357$ & $114.5$ & $147$ & $53.4$ & $1018$ & $155.1$ & $170$ & $49.4$ & $679$ & $236.3$ & $186$ & $43.3$ \\
Uniform & K/V, output & $\mathbf{1242}$ & $\mathbf{125.1}$ & $144$ & $54.0$ & $\mathbf{932}$ & $\mathbf{169.4}$ & $165$ & $49.3$ & $\mathbf{622}$ & $\mathbf{257.9}$ & $179$ & $42.9$ \\
Measured KL & Q/K/V, input & $1822$ & $85.3$ & $\mathbf{180}$ & $56.4$ & $1701$ & $92.8$ & $\mathbf{187}$ & $49.3$ & $1282$ & $125.2$ & $\mathbf{201}$ & $42.9$ \\
Measured KL & K/V, output & $2048$ & $75.9$ & $176$ & $\mathbf{56.9}$ & $2048$ & $77.1$ & $182$ & $\mathbf{50.9}$ & $1920$ & $83.6$ & $193$ & $\mathbf{44.2}$ \\
\bottomrule
\end{tabular}
}

\end{table}

\clearpage
\newpage
\section{Compression Cost}
\label{app:cost}
\begin{figure}[H]
\centering
\includegraphics[width=\linewidth]{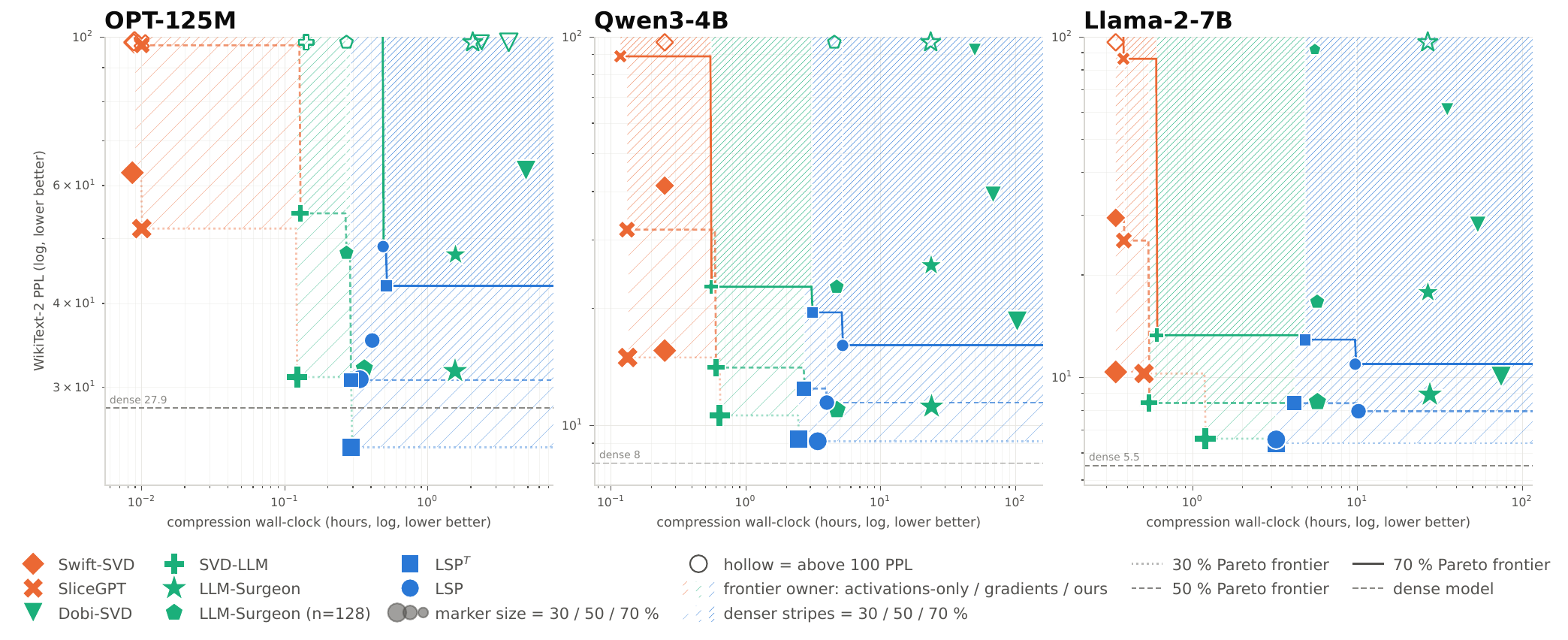}
\caption{\textbf{Compression cost with measured-KL allocation on three LLMs.} WikiText-2 perplexity at $-30/50/70\%$ compression (marker size) against wall-clock hours. \method costs include initialization, training through the selected epoch, and the full KL measurement on the 7-point grid on OPT-125M and the 15-point grid on Qwen3-4B and Llama-2-7B.}
\label{fig:cost}
\end{figure}

\Cref{fig:cost} compares \method's compression cost with that of activation-based methods (SliceGPT, Swift-SVD) and loss-aware ones (LLM-Surgeon, Dobi-SVD, SVD-LLM), on OPT-125M, Qwen3-4B and Llama-2-7B.
Giving the loss-aware baselines more budget does not close the gap. LLM-Surgeon calibrated on \method's $1024$ sequences costs $4$--$6\times$ more than at its default $128$, and its perplexity does not improve. Dobi-SVD, trained for $20$ epochs on the same $1024$ samples, costs $2$--$103$ wall-clock hours and stays far behind: $61.2$ against $10.9$ perplexity on Llama-2-7B at $-70\%$.

Compression cost is paid once per model and ratio, and it is small next to the cost of serving the compressed model. We argue it should therefore not drive the choice between methods, whereas accuracy and inference efficiency (\Cref{subsec:compute}) are paid on every query.
Even so, \method is Pareto-optimal at higher budgets: at every model--ratio setting, one of its two objectives gives the lowest perplexity, within $0.5$--$10$ wall-clock hours. Only methods that are cheaper and less accurate share the frontier: SVD-LLM with LoRA recovery, SliceGPT and Swift-SVD.
\method$^T$ converges faster than \method, reaching its validation optimum at up to $2.5\times$ lower cost ($4.1$ against $10.1$ hours on Llama-2-7B at $-50\%$), though at a higher perplexity at $-50$ and $-70\%$.
The KL measurement is a substantial fixed cost: $0.17$, $1.74$ and $1.11$ wall-clock hours on OPT-125M, Qwen3-4B and Llama-2-7B. The \method points use individual-run perplexities, which can differ from the three-seed means of \Cref{tab:results}.

\section{Transfer on the Growing Calibration Pool}
\label{app:vit_transfer}

\Cref{tab:vit_transfer} breaks down the ViT transfer results of \Cref{tab:vit_results} by calibration pool and target.
The pools grow from CIFAR-100 alone ($10$k images) to the six-dataset pool of \Cref{tab:vit_results} ($47$k), with the same images for every method. \method's mean transfer rises with every added dataset at every ratio, and on the full pool it is best on all three targets at every ratio; the baselines gain less and not always monotonically (FLAR-SVD at $-50\%$ loses $4.4$ points when CIFAR-10 and EuroSAT are added).
The last row of each method is CIFAR-100 alone at the size of the full pool. For \method, this extra CIFAR-100 data changes mean transfer by $+1.2$, $-0.6$ and $+0.8$ points at $-30/50/70\%$, whereas the diverse pool exceeds this size-matched control by $3.7$, $7.1$ and $6.6$ points (the Gain of \Cref{tab:vit_results}): the gain comes from pool composition, not size.

\begin{table}[H]
\captionsetup{font=small}
\centering
\centering
\caption{Complete downstream-transfer results for compressed ViT-B/16, for every calibration pool and target. \emph{Compressed on}: the datasets in the pool (\ding{51}) and its size, the same pools for every method; \ding{51}$^\dagger$ is CIFAR-100 alone at the size of the six-dataset pool. \emph{Transfer to}: linear-probe accuracy on Pets, Aircraft and Places365, absent from every pool, and their mean; \emph{Source}: CIFAR-100 accuracy through the original head. \method$^T$ needs CIFAR-100 labels. Best per column, pool and ratio in \textbf{bold}.}
\vspace{-0.3cm}
\label{tab:vit_transfer}
\scriptsize
\setlength{\tabcolsep}{2.5pt}
\renewcommand{\arraystretch}{0.86}
\begin{tabular}{@{}l@{\hspace{3pt}}c@{\hspace{5pt}}cccccc@{\hspace{0.4em}}r@{\hspace{0.8em}}cccc@{\hspace{0.8em}}c@{}}
\toprule
 & & \multicolumn{7}{c}{\textbf{Compressed on}} & \multicolumn{4}{c}{\textbf{Transfer to}} & \textbf{Source} \\
\cmidrule(lr){3-9}\cmidrule(lr){10-13}\cmidrule(lr){14-14}
\textbf{Method} & & \rotatebox{90}{CIFAR-100} & \rotatebox{90}{Food-101} & \rotatebox{90}{CIFAR-10} & \rotatebox{90}{EuroSAT} & \rotatebox{90}{STL-10} & \rotatebox{90}{DTD} & Images & Pets & Aircraft & Places365 & Mean & CIFAR-100 \\
\midrule
Dense (Base) & & \multicolumn{7}{c}{---} & 71.9 & 36.9 & 33.5 & 47.4 & 89.9 \\
\midrule
\multicolumn{14}{@{}l}{\textit{-30\% parameters}} \\
\multirow{5}{*}{FLAR-SVD} & \multirow{5}{*}{\rule{0.5pt}{32pt}} & \ding{51} &  &  &  &  &  & 10k & \textbf{67.1} & 32.3 & \textbf{31.2} & \textbf{43.6} & 88.7 \\
 & & \ding{51} & \ding{51} &  &  &  &  & 20k & 68.6 & 34.1 & \textbf{32.2} & 45.0 & 88.8 \\
 & & \ding{51} & \ding{51} & \ding{51} & \ding{51} &  &  & 40k & 68.6 & 33.2 & \textbf{32.7} & 44.9 & 88.7 \\
 & & \ding{51} & \ding{51} & \ding{51} & \ding{51} & \ding{51} & \ding{51} & 47k & 68.4 & 33.3 & 32.4 & 44.7 & 88.6 \\
 & & \ding{51}\smash{$^\dagger$} &  &  &  &  &  & 47k & 66.9 & 32.7 & \textbf{31.1} & 43.6 & 88.8 \\
\addlinespace[1pt]
\multirow{5}{*}{SVD-LLM (W)} & \multirow{5}{*}{\rule{0.5pt}{32pt}} & \ding{51} &  &  &  &  &  & 10k & 63.2 & 31.1 & 29.5 & 41.2 & 88.2 \\
 & & \ding{51} & \ding{51} &  &  &  &  & 20k & 63.3 & 32.3 & 28.8 & 41.5 & 87.9 \\
 & & \ding{51} & \ding{51} & \ding{51} & \ding{51} &  &  & 40k & 64.4 & 32.2 & 29.2 & 41.9 & 87.8 \\
 & & \ding{51} & \ding{51} & \ding{51} & \ding{51} & \ding{51} & \ding{51} & 47k & 64.7 & 32.6 & 29.5 & 42.3 & 87.7 \\
 & & \ding{51}\smash{$^\dagger$} &  &  &  &  &  & 47k & 63.0 & 31.2 & 29.2 & 41.2 & 88.2 \\
\addlinespace[1pt]
\multirow{5}{*}{PELA} & \multirow{5}{*}{\rule{0.5pt}{32pt}} & \ding{51} &  &  &  &  &  & 10k & 66.2 & \textbf{32.9} & 30.1 & 43.1 & 89.2 \\
 & & \ding{51} & \ding{51} &  &  &  &  & 20k & 65.0 & 32.2 & 30.3 & 42.5 & 89.1 \\
 & & \ding{51} & \ding{51} & \ding{51} & \ding{51} &  &  & 40k & 67.3 & 34.1 & 30.1 & 43.8 & 88.9 \\
 & & \ding{51} & \ding{51} & \ding{51} & \ding{51} & \ding{51} & \ding{51} & 47k & 67.1 & 34.2 & 30.5 & 43.9 & 88.9 \\
 & & \ding{51}\smash{$^\dagger$} &  &  &  &  &  & 47k & 66.0 & 31.1 & 29.8 & 42.3 & 89.2 \\
\addlinespace[1pt]
\textbf{LSP$^T$ (Ours)} & \rule{0.5pt}{5.5pt} & \ding{51}\smash{$^\dagger$} &  &  &  &  &  & 47k & \textbf{67.8} & \textbf{33.9} & 30.8 & \textbf{44.2} & 89.3 \\
\addlinespace[1pt]
\multirow{5}{*}{\textbf{LSP (Ours)}} & \multirow{5}{*}{\rule{0.5pt}{32pt}} & \ding{51} &  &  &  &  &  & 10k & 66.0 & 32.4 & 29.8 & 42.7 & \textbf{89.3} \\
 & & \ding{51} & \ding{51} &  &  &  &  & 20k & \textbf{70.6} & \textbf{36.8} & 32.1 & \textbf{46.5} & \textbf{89.3} \\
 & & \ding{51} & \ding{51} & \ding{51} & \ding{51} &  &  & 40k & \textbf{71.3} & \textbf{37.2} & 32.0 & \textbf{46.8} & \textbf{89.2} \\
 & & \ding{51} & \ding{51} & \ding{51} & \ding{51} & \ding{51} & \ding{51} & 47k & \textbf{71.7} & \textbf{37.7} & \textbf{33.2} & \textbf{47.5} & \textbf{89.2} \\
 & & \ding{51}\smash{$^\dagger$} &  &  &  &  &  & 47k & 67.3 & 33.8 & 30.5 & 43.9 & \textbf{89.4} \\
\midrule
\multicolumn{14}{@{}l}{\textit{-50\% parameters}} \\
\multirow{5}{*}{FLAR-SVD} & \multirow{5}{*}{\rule{0.5pt}{32pt}} & \ding{51} &  &  &  &  &  & 10k & 58.1 & 27.2 & \textbf{28.7} & 38.0 & 86.2 \\
 & & \ding{51} & \ding{51} &  &  &  &  & 20k & 61.1 & 30.8 & \textbf{29.7} & 40.5 & 86.1 \\
 & & \ding{51} & \ding{51} & \ding{51} & \ding{51} &  &  & 40k & 54.3 & 27.0 & 27.1 & 36.1 & 83.5 \\
 & & \ding{51} & \ding{51} & \ding{51} & \ding{51} & \ding{51} & \ding{51} & 47k & 54.9 & 26.9 & 26.7 & 36.2 & 83.1 \\
 & & \ding{51}\smash{$^\dagger$} &  &  &  &  &  & 47k & 58.6 & 27.5 & \textbf{28.3} & 38.1 & 86.4 \\
\addlinespace[1pt]
\multirow{5}{*}{SVD-LLM (W)} & \multirow{5}{*}{\rule{0.5pt}{32pt}} & \ding{51} &  &  &  &  &  & 10k & 58.6 & 28.3 & 26.6 & 37.8 & 86.1 \\
 & & \ding{51} & \ding{51} &  &  &  &  & 20k & 59.6 & 29.1 & 27.0 & 38.6 & 84.9 \\
 & & \ding{51} & \ding{51} & \ding{51} & \ding{51} &  &  & 40k & 59.6 & 28.9 & 27.2 & 38.6 & 84.4 \\
 & & \ding{51} & \ding{51} & \ding{51} & \ding{51} & \ding{51} & \ding{51} & 47k & 60.1 & 28.4 & 27.1 & 38.6 & 84.1 \\
 & & \ding{51}\smash{$^\dagger$} &  &  &  &  &  & 47k & 58.4 & 28.7 & 27.2 & 38.1 & 86.2 \\
\addlinespace[1pt]
\multirow{5}{*}{PELA} & \multirow{5}{*}{\rule{0.5pt}{32pt}} & \ding{51} &  &  &  &  &  & 10k & \textbf{62.6} & \textbf{29.7} & 28.5 & \textbf{40.3} & 87.9 \\
 & & \ding{51} & \ding{51} &  &  &  &  & 20k & 61.6 & 30.9 & 28.1 & 40.2 & \textbf{87.7} \\
 & & \ding{51} & \ding{51} & \ding{51} & \ding{51} &  &  & 40k & 64.6 & 31.3 & 29.3 & 41.7 & 87.0 \\
 & & \ding{51} & \ding{51} & \ding{51} & \ding{51} & \ding{51} & \ding{51} & 47k & 63.6 & 32.3 & 29.3 & 41.8 & \textbf{87.8} \\
 & & \ding{51}\smash{$^\dagger$} &  &  &  &  &  & 47k & \textbf{62.0} & \textbf{30.0} & 28.2 & \textbf{40.1} & \textbf{88.5} \\
\addlinespace[1pt]
\textbf{LSP$^T$ (Ours)} & \rule{0.5pt}{5.5pt} & \ding{51}\smash{$^\dagger$} &  &  &  &  &  & 47k & 56.0 & 23.7 & 25.7 & 35.1 & 87.8 \\
\addlinespace[1pt]
\multirow{5}{*}{\textbf{LSP (Ours)}} & \multirow{5}{*}{\rule{0.5pt}{32pt}} & \ding{51} &  &  &  &  &  & 10k & 58.9 & 28.1 & 27.5 & 38.2 & \textbf{88.2} \\
 & & \ding{51} & \ding{51} &  &  &  &  & 20k & \textbf{63.2} & \textbf{31.7} & 28.5 & \textbf{41.1} & 87.5 \\
 & & \ding{51} & \ding{51} & \ding{51} & \ding{51} &  &  & 40k & \textbf{67.5} & \textbf{34.4} & \textbf{30.4} & \textbf{44.1} & \textbf{88.0} \\
 & & \ding{51} & \ding{51} & \ding{51} & \ding{51} & \ding{51} & \ding{51} & 47k & \textbf{68.2} & \textbf{35.2} & \textbf{30.7} & \textbf{44.7} & 87.7 \\
 & & \ding{51}\smash{$^\dagger$} &  &  &  &  &  & 47k & 58.7 & 26.9 & 27.3 & 37.6 & 88.3 \\
\midrule
\multicolumn{14}{@{}l}{\textit{-70\% parameters}} \\
\multirow{5}{*}{FLAR-SVD} & \multirow{5}{*}{\rule{0.5pt}{32pt}} & \ding{51} &  &  &  &  &  & 10k & 37.7 & 18.5 & 22.7 & 26.3 & 70.7 \\
 & & \ding{51} & \ding{51} &  &  &  &  & 20k & 40.5 & 19.3 & 23.7 & 27.8 & 68.0 \\
 & & \ding{51} & \ding{51} & \ding{51} & \ding{51} &  &  & 40k & 41.2 & 20.0 & 23.3 & 28.1 & 66.9 \\
 & & \ding{51} & \ding{51} & \ding{51} & \ding{51} & \ding{51} & \ding{51} & 47k & 40.2 & 19.8 & 23.6 & 27.9 & 64.6 \\
 & & \ding{51}\smash{$^\dagger$} &  &  &  &  &  & 47k & 38.5 & 18.9 & 22.5 & 26.6 & 71.5 \\
\addlinespace[1pt]
\multirow{5}{*}{SVD-LLM (W)} & \multirow{5}{*}{\rule{0.5pt}{32pt}} & \ding{51} &  &  &  &  &  & 10k & 47.9 & \textbf{23.2} & 23.9 & 31.7 & 77.2 \\
 & & \ding{51} & \ding{51} &  &  &  &  & 20k & 50.9 & 24.6 & 24.0 & 33.2 & 70.5 \\
 & & \ding{51} & \ding{51} & \ding{51} & \ding{51} &  &  & 40k & 50.5 & 23.7 & 24.1 & 32.8 & 68.2 \\
 & & \ding{51} & \ding{51} & \ding{51} & \ding{51} & \ding{51} & \ding{51} & 47k & 51.9 & 24.0 & 24.4 & 33.4 & 66.2 \\
 & & \ding{51}\smash{$^\dagger$} &  &  &  &  &  & 47k & 48.2 & \textbf{23.0} & 23.9 & 31.7 & 77.1 \\
\addlinespace[1pt]
\multirow{5}{*}{PELA} & \multirow{5}{*}{\rule{0.5pt}{32pt}} & \ding{51} &  &  &  &  &  & 10k & \textbf{53.1} & 22.4 & \textbf{27.1} & \textbf{34.2} & 84.4 \\
 & & \ding{51} & \ding{51} &  &  &  &  & 20k & 54.3 & 24.3 & \textbf{27.5} & 35.4 & \textbf{84.0} \\
 & & \ding{51} & \ding{51} & \ding{51} & \ding{51} &  &  & 40k & 56.2 & 26.3 & 28.1 & 36.9 & 83.0 \\
 & & \ding{51} & \ding{51} & \ding{51} & \ding{51} & \ding{51} & \ding{51} & 47k & 58.9 & 27.8 & 28.2 & 38.3 & 82.8 \\
 & & \ding{51}\smash{$^\dagger$} &  &  &  &  &  & 47k & \textbf{53.0} & 22.3 & \textbf{26.9} & \textbf{34.1} & 85.3 \\
\addlinespace[1pt]
\textbf{LSP$^T$ (Ours)} & \rule{0.5pt}{5.5pt} & \ding{51}\smash{$^\dagger$} &  &  &  &  &  & 47k & 46.4 & 20.3 & 22.7 & 29.8 & 84.2 \\
\addlinespace[1pt]
\multirow{5}{*}{\textbf{LSP (Ours)}} & \multirow{5}{*}{\rule{0.5pt}{32pt}} & \ding{51} &  &  &  &  &  & 10k & 49.7 & 22.2 & 24.8 & 32.2 & \textbf{85.8} \\
 & & \ding{51} & \ding{51} &  &  &  &  & 20k & \textbf{57.1} & \textbf{27.6} & 27.2 & \textbf{37.3} & 83.6 \\
 & & \ding{51} & \ding{51} & \ding{51} & \ding{51} &  &  & 40k & \textbf{58.6} & \textbf{26.9} & \textbf{29.6} & \textbf{38.3} & \textbf{83.9} \\
 & & \ding{51} & \ding{51} & \ding{51} & \ding{51} & \ding{51} & \ding{51} & 47k & \textbf{60.4} & \textbf{29.1} & \textbf{29.4} & \textbf{39.6} & \textbf{84.6} \\
 & & \ding{51}\smash{$^\dagger$} &  &  &  &  &  & 47k & 49.4 & 22.9 & 26.6 & 33.0 & \textbf{85.9} \\
\bottomrule
\end{tabular}

\end{table}

\end{document}